%% file: main.tex
\documentclass{article}

\PassOptionsToPackage{numbers, compress}{natbib}

\usepackage[preprint]{neurips_2026}

\usepackage[utf8]{inputenc} 
\usepackage[T1]{fontenc}    
\usepackage{hyperref}       
\usepackage{url}            
\usepackage{booktabs}       
\usepackage{amsfonts}       
\usepackage{nicefrac}       
\usepackage{microtype}      
\usepackage{xcolor}         
\usepackage{adjustbox}
\usepackage{amsmath}
\usepackage{amssymb}
\usepackage{mathtools}
\usepackage{amsthm}

\usepackage{algorithm}
\usepackage{algorithmic}

\usepackage{graphicx}
\usepackage{subfigure}
\usepackage{multirow} 

\usepackage[capitalize]{cleveref}
\usepackage{fontawesome5}
\usepackage{xurl}
\usepackage{enumitem}
\usepackage{threeparttable} 

\theoremstyle{plain}
\newtheorem{theorem}{Theorem}[section]

\newtheorem{corollary}[theorem]{Corollary}
\theoremstyle{definition}

\theoremstyle{remark}
\newtheorem{remark}[theorem]{Remark}

\newcommand{\R}{\mathbb{R}}

\def\objname{\textsf{DRIO}}

\title{Multivariate Time Series Data Imputation via Distributionally Robust Regularization}

\author{%
  Che-Yi Liao\textsuperscript{1}\thanks{email: cliao48@gatech.edu} 
  \hspace{1em} 
  Zheng Dong\textsuperscript{1}
  \hspace{1em} 
  Gian Gabriel Garcia\textsuperscript{2}
  \hspace{1em} 
  Kamran Paynabar\textsuperscript{1} \\
  \textsuperscript{1}Georgia Institute of Technology, Atlanta, GA, USA\\
  \textsuperscript{2}University of Washington, Seattle, WA, USA\\
}

\begin{document}

\maketitle

\begin{abstract}
Multivariate time series imputation is often compromised by mismatch between the observed and true data distributions, a bias induced by the combined effects of time-series non-stationarity and systematic missingness. Standard methods that encourage point-wise reconstruction or direct distributional alignment may overfit these biased observations. We propose the Distributionally Robust Regularized Imputer Objective (DRIO), which jointly minimizes reconstruction error and the worst-case divergence between the imputer distribution and data distributions within a Wasserstein ambiguity set. We derive a tractable upper-bound surrogate that reduces infinite-dimensional optimization over measures to adversarial search over sample trajectories, and develop an alternating learning algorithm compatible with modern deep learning backbones. Comprehensive experiments on diverse real-world datasets show that DRIO consistently provides robust imputation and suggests improved downstream forecasting under various missingness scenarios.
\end{abstract}

\input{introduction}
\input{method}
\input{numerical}
\input{conclusion}

\newpage
\bibliographystyle{unsrtnat}
\bibliography{references}
\input{appendix} 



\end{document}

%% file: introduction.tex
\section{Introduction}\label{sec:introduction}


Multivariate time series (MTS) data encode rich spatiotemporal dependencies that are critical for downstream tasks including forecasting, anomaly detection, and decision-making across real-world applications such as healthcare operations, traffic monitoring, industrial sensor networks, among others \citep{ghosh2009multivariate, kirchgassner2012introduction, hamilton2020time, yang2025responsible, yang2026development, liao2025constraint}. However, real-world time series measurements are frequently incomplete due to sensor failures, communication dropouts, or resource constraints, which highlights the importance of robust imputation methods to recover missing entries before subsequent analysis and decision-making \citep{
li2020spatiotemporal,
liao2022racial,
emmanuel2021survey, 
liao2025estimating, 
cai2025missing}.

In this MTS imputation task, the primary challenge is the \textbf{mismatch between the true data-generating process and the observed empirical distribution}. In other words, the observed data distribution $\widehat{\mathbb{P}}_N$ can be a biased estimator of the true data-generation distribution $\mathbb{P}_{\mathrm{true}}$, causing classical imputation methods to fail as they fit to the observed entries. This bias stems from two intertwined factors: (a) time series data non-stationarity and (b) systematic missingness patterns.

First, \textit{time series data non-stationarity} is a common phenomenon wherein the underlying data-generating process evolves over either temporal or sample space \citep{sayed2012learning,cheng2015time,ditzler2015learning,liao2026tides}. For example, in traffic modeling, weekday and weekend traffic patterns may exhibit markedly different statistical dependencies while different sensor locations can show distinct traffic flow patterns. Consequently, reliable estimation of the data distribution becomes progressively more challenging as the feature-temporal dimensionality increases and the non-stationary dynamics grow more complex \citep{dixit2023contemporary,tran2019sample,hsiao2025balancing}. 

This distribution estimation complexity can further be exaggerated by \textit{non-uniform missingness pattern} across MTS data entries. While measurements may be missing completely at random (MCAR) with each entry having an identical and independent probability of being unobserved, missing not at random (MNAR) are particularly prevalent in practice \citep{bollinger2014trouble, pham2022missing,cai2025missing}. Under MNAR, missingness can depend on the sample heterogeneity, feature values, and time, which causes specific regions of the data manifold to remain unobserved, introducing another layer of complexity in MTS imputation.

Existing time series imputation methods rarely \textit{explicitly} address this mismatch between empirical distribution and true data-generating process. Rather, they typically implicitly assume minimizing reconstruction error on observed entries or the divergence between imputers and the empirical distributions will generalize to missing entries and other samples \cite{zhang2025missing, afrifa2020missing, cao2018brits, shen2023bidirectional, mattei2019miwae, ipsen2020not, muzellec2020missing}. However, this assumption can fail when empirical distribution is a biased sample of the true data manifold since they do not directly incorporate \textit{bias or uncertainty} of distributional alignment into imputation.

\begin{figure*}[t]
\begin{center}
\includegraphics[width=\textwidth]{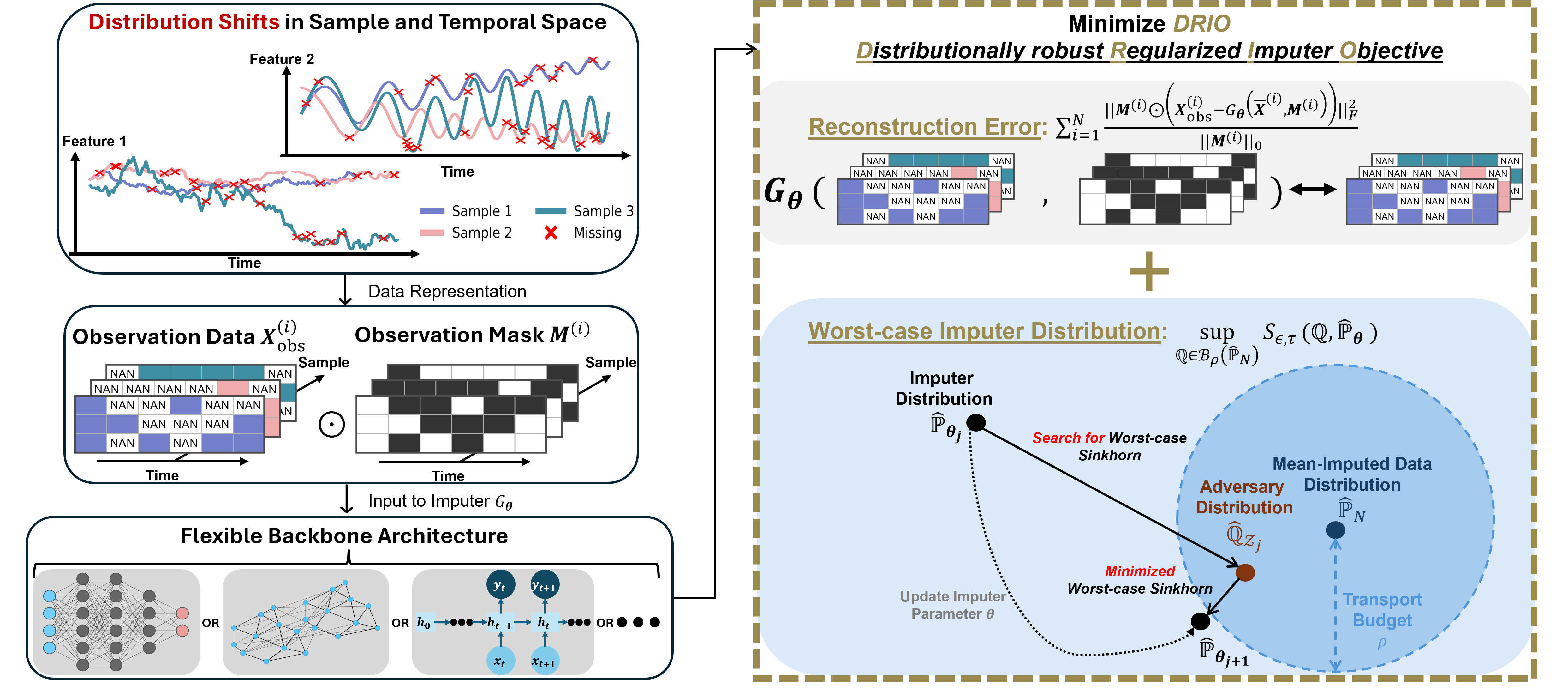}
\caption{Overview of the Multivariate Time Series Imputation framework under $\objname$.}
\label{fig:drio_framework}
\end{center}
\vskip -0.2in
\end{figure*}

To mitigate this distributional uncertainty around observed data distribution, we propose a novel \textbf{D}istributionally robust \textbf{R}egularized \textbf{I}mputer \textbf{O}bjective ($\objname$), which enables the imputer to hedge its predictions against worst-case distributional alignment within a Wasserstein ambiguity set centered at the empirical measures (\cref{fig:drio_framework}). $\objname$ acts as a regularizer that controls the trade-off between point-wise reconstruction accuracy and robustness of the imputer by minimizing the risk against an adversarial distribution that exploits the support bias of the observed data. $\objname$ is formulated as a training objective that can be implemented on top of a broad class of differentiable imputer backbones to enhance their robustness to potentially biased data, while its empirical performance may depend on the backbone design.

\textbf{Contribution.} Our main contributions are summarized as follows:
\begin{enumerate}[label=(\roman*)]
\item We are among the first to \textbf{formalize the non-stationary MTS imputation under distribution shifts} by explicitly addressing the distributional mismatch between observed data and the true data-generating process.
\item We propose a novel distributionally robust regularizer ($\objname$) that \textbf{balances reconstruction accuracy and worst-case distributional shifts}. Moreover, we \textbf{derive a tractable upper-bound surrogate} (Theorem~\ref{thm:dual-formulation} and Corollary~\ref{cor:fixed-gamma-drio}) that reduces the infinite-dimensional optimization over measures to an adversarial search over samples.
\item We \textbf{develop an efficient alternating optimization algorithm} (\cref{alg:dro_training}) that jointly updates adversarial trajectories and imputer parameters, enabling end-to-end training with differentiable model backbones.
\item We provide extensive empirical evidence across seven diverse real-world datasets showing that DRIO \textbf{achieves robust imputation and downstream forecasting performance under both MCAR and MNAR mechanisms at varying missing ratios} (10\%/50\%/90\%).
\end{enumerate}

\textbf{Related Work.}
MTS imputation methods can be broadly grouped into \textit{point-wise reconstruction} methods and \textit{distributional alignment} methods.

\textit{Point-wise Reconstruction Methods.}
A large body of work designs neural architectures to reconstruct missing entries from observed temporal and cross-feature dependencies. Early methods rely on recurrent dynamics and consistency losses \citep{cao2018brits}, while more recent approaches use attention, graph neural networks, convolutional modules, and Transformer-style backbones to capture long-range temporal patterns and spatial correlations in MTS data \citep{,suo2020glima,ma2020remian,li2020spatiotemporal,ye2021spatial,marisca2022learning,shen2023bidirectional,du2023saits,nie2024imputeformer}. While these methods have become increasingly expressive, they typically rely on point-wise reconstruction objectives, which when used alone can overfit the empirical observation pattern and become less robust when the observed distribution is biased by non-stationarity or systematic missingness.

\textit{Distributional and Generative Imputation.}
Another line of work aims to model the conditional distribution of missing values rather than only produce point estimates. Representative techniques include adversarial learning, variational inference, score-based diffusion, structured state-space diffusion, and missingness-aware generative modeling for MNAR settings~\citep{yoon2018gain,mattei2019miwae,ipsen2020not,Li2021SCVAE,tashiro2021csdi,alcaraz2022diffusion,gao2024diffimp}. 
Additionally, optimal transport (OT) theory has recently been used to align imputed and observed distributions through Wasserstein or Sinkhorn-type objectives, including extensions to temporal and frequency-domain structures~\citep{muzellec2020missing,huang2025temporalwassersteinimputationversatile,wang2025optimal}.
Overall, these methods explicitly incorporate uncertainty into modeling imputers and can provide geometrically meaningful alignment objectives, but usually treat the empirical observed distribution as an unbiased proxy for the true data-generating process, which may not prevent aligning to the empirical (and potentially biased) observed distribution.

%% file: method.tex
\section{Method}
\label{sec:method}
\textbf{Background and Notation.}
\label{sec:problem-setting}
We denote $\mathbb{P}_{\text{true}}$ as the probability distribution underlying the data-generating process that describes joint evolution of $D$ correlated features (e.g., sensors, spatial nodes) over time. We assume this process is observed at discrete timestamps $t \in \{1, \dots, T\}$ and a single realization forms a ground truth trajectory $\boldsymbol{X} \in \mathbb{R}^{D \times T}$. Since $\mathbb{P}_{\text{true}}$ captures complex time-varying dependencies on the feature-temporal space, the marginal distributions of features may shift significantly over the temporal horizon.

With the definition of data-generating process, we denote  $\{\boldsymbol{X}^{(1)}, \dots, \boldsymbol{X}^{(N)}\} \in \mathbb{R}^{N\times D\times T}$ as $N$ trajectories drawn from $\mathbb{P}_{\text{true}}$. To account for missing entries, we define an observation mask as $\boldsymbol{M}^{(i)} \in \{0, 1\}^{D \times T}$, which is a realization from a random matrix such that $\boldsymbol{M}^{(i)}_{d,t} = 1$ if $\boldsymbol{X}^{(i)}_{d,t}$ is observed and $0$ otherwise. Consequently, the $i$\textsuperscript{th} sample trajectory is denoted by $\boldsymbol{X}^{(i)}_\mathrm{obs}\coloneqq\boldsymbol{X}^{(i)}\odot \boldsymbol{M}^{(i)}$, where $\odot$ is the Hadamard product (See top left of Figure~\ref{fig:drio_framework}). Moreover, we define the empirical mean value for feature $d$ at a time $t$ as $\overline{x}_{d,t}:=  \sum_{i}\boldsymbol{X}^{(i)}_{\mathrm{obs},d,t}/\sum_j\boldsymbol{M}^{(j)}_{d,t}$. We set $\overline{x}_{d,t} = 0$ if $\sum_j\boldsymbol{M}^{(j)}_{d,t} = 0$. Consequently, the $i$\textsuperscript{th}
mean-imputed sample is defined as $\overline{\boldsymbol{X}}^{(i)}$ where $\overline{\boldsymbol{X}}^{(i)}_{d,t}:=\boldsymbol{X}^{(i)}_{\mathrm{obs},d,t} + (1-\boldsymbol{M}^{(i)}_{d,t}) \overline{x}_{d,t}$. Moreover, $\widehat{\boldsymbol{X}}^{(i)}\coloneqq \boldsymbol{X}^{(i)}_{\mathrm{obs}} + (\boldsymbol{1} - \boldsymbol{M}^{(i)} )G_{\boldsymbol{\theta}}(\boldsymbol{X}^{(i)}_{\mathrm{obs}}, \boldsymbol{M}^{(i)}))$ is the imputed data by imputer $G_{\boldsymbol{\theta}}$ parametrized by $\boldsymbol{\theta}$.

\textbf{Sinkhorn Divergence.}
Let $\mu, \nu \in \mathcal{P}(\mathbb{R}^{D \times T})$ be two probability measures, where $\mathcal{P}(S)$ denotes the collection of all distributions supported on space $S$. The entropic unbalanced transport cost is defined as
$
    W_{\epsilon, \tau}(\mu, \nu) \coloneqq \inf_{\pi \geq 0} \{ \int \|\boldsymbol{x} - \boldsymbol{z}\|_F^2 \, \mathrm{d}\pi + \epsilon \, \mathrm{KL}(\pi \,|\, \mu \otimes \nu) + \tau ( \mathrm{KL}(\pi_1 \,|\, \mu) + \mathrm{KL}(\pi_2 \,|\, \nu) ) \},
$
where $\pi$ is a positive measure with marginals $\pi_1$ and $\pi_2$, $\epsilon > 0$ controls entropic regularization, and $\tau > 0$ governs marginal relaxation that allows mass creation and destruction. We use the standard KL relaxation in our formulation, which leads to the classical entropic unbalanced Sinkhorn updates. 
Other marginal divergences, such as $\chi^2$, could also be used but require different update rules and are left for future work
\citep{sejourne2019sinkhorn, choi2023generative}. The \textbf{Unbalanced Sinkhorn Divergence} is then:
$
    S_{\epsilon, \tau}(\mu, \nu) \coloneqq W_{\epsilon, \tau}(\mu, \nu) - \tfrac{1}{2} ( W_{\epsilon, \tau}(\mu, \mu) + W_{\epsilon, \tau}(\nu, \nu) )
$. 

While Wasserstein-$p$ distance is a natural choice for measuring distributional discrepancy, its exact computation requires solving a linear program whose optimal value is generally non-smooth with respect to the imputer parameters $\boldsymbol{\theta}$ through the imputed samples $\widehat{\boldsymbol{X}}^{(i)}$. Therefore, it is common to adopt the Sinkhorn Divergence since the entropic regularization enables efficient GPU-parallelizable computation and produces a fully differentiable objective, while the debiasing formula further corrects the entropic bias, ensuring that $S_{\epsilon,\tau}(\mu, \mu) = 0$ \cite{muzellec2020missing, wang2025optimal}. Moreover, Sinkhorn Divergence provides meaningful gradients even when distribution supports are disjoint, which is a robust alternative compared to other likelihood-based metrics, e.g., KL divergence, and point-wise metrics, e.g., least-square \citep{villani2008optimal,sejourne2019sinkhorn,wang2025sinkhorn}. Notably, the \textit{Unbalanced} Sinkhorn Divergence relaxes the strict mass conservation of standard optimal transport, allowing local mass creation and destruction via a soft penalty \citep{sejourne2019sinkhorn, choi2023generative}. This prevents outliers from dominating gradient updates (whether this outlier is from the ambiguity set or from poorly performing imputer predictions), thereby ensuring training stability. We expand the discussions of our design choices in Appendix \S\ref{sec:discuss-loss}.

\subsection{Distributionally Robust Imputation Regularizer}\label{sec:primal}
Our robust imputation objective ($\objname$) jointly minimizes \textit{point-wise reconstruction accuracy} of the observed data and \textit{worst-case divergence} of the empirical distribution to its neighbors. In this section, we first detail the worst-case divergence problem, then provide the primal view of the proposed robust imputation objective.

\textbf{Ambiguity Set Construction.}
Since raw data has missing entries, we construct a Wasserstein ambiguity set centered at the empirical distribution of the \textit{mean-imputed data} (Bottom right of Figure~\ref{fig:drio_framework}). Conceptually, this set contains all distributions that could plausibly have generated the observed data, with the set radius (transport budget) controlling the degree of distributional uncertainty and thus the robustness level. Specifically, we define the \textit{empirical measure} as $\widehat{\mathbb{P}}_N\coloneqq \sum_i\delta_{\overline{\boldsymbol{X}}^{(i)}}/N$ and the \textit{ground cost} as
$
    c_{\boldsymbol{X}}(\boldsymbol{Z}) = \| \boldsymbol{X} - \boldsymbol{Z} \|_F^2,
$
where $\boldsymbol{X}, \boldsymbol{Z}\in\mathbb{R}^{D\times T}$.
The ambiguity set $\mathcal{B}_\rho(\widehat{\mathbb{P}}_N)$ is then defined as the set of all probability measures $\mathbb{Q}$ supported on a subset of $\mathbb{R}^{D \times T}$ satisfying a transport budget constraint:
\begin{equation}\label{eq:ambiguity-set}
\mathcal{B}_\rho (\widehat{\mathbb{P}}_N) \coloneqq 
\Big\{\mathbb{Q} \in \mathcal{P}(\mathbb{R}^{D \times T}):
\inf_{\pi \in \Pi(\widehat{\mathbb{P}}_N, \mathbb{Q})} \mathbb{E}_{(\boldsymbol{X}, \boldsymbol{Z}) \sim \pi} \left[c_{\boldsymbol{X}}\left( \boldsymbol{Z} \right) \right] \leq \rho \Big\},
\end{equation}
where $\Pi(\widehat{\mathbb{P}}_N, \mathbb{Q})$ denotes the set of couplings between the empirical and candidate distributions, and $\rho \geq 0$ is the radius of the uncertainty set representing the transport budget.

\begin{remark}[Ambiguity Set Geometry]\label{remark:ambiguity-set}
We construct our ambiguity set \eqref{eq:ambiguity-set} using the Wasserstein metric centered on the empirical measure of the mean-imputed data for practical considerations specific to our imputation problem. First, centering the ambiguity set at the empirical distribution of the mean-imputed data $\{\overline{\boldsymbol{X}}^{(i)}\}$ allows exploration of the missing entries starting from the mean, subject to the transport budget $\rho$. Moreover, the empirical measure $\widehat{\mathbb{P}}_N$ is a discrete collection of point masses from the mean-imputed data, which are realizations from a potentially continuous ground-truth distribution $\mathbb{P}_{\mathrm{true}}$. With the Wasserstein ambiguity set definition, we explicitly accommodate continuity in the support of the learned worst-case distribution, effectively creating a continuous neighborhood around the empirical distribution. Finally, this formulation enables the imputer to learn the correlation structure of the distribution without being constrained by summary statistics (e.g., moments) that may be biased estimates of the ground truth.
\qed
\end{remark}
\textbf{Primal Objective.}
With the ambiguity set underlying the observed data, we now formulate the primal $\objname$ objective, which jointly minimizes (i) the reconstruction error of the imputer, and (ii) the distributional alignment between the worst-case data distribution and the imputer-generated distribution \textit{over ambiguity set}. To measure distributional discrepancy, we adopt the \textit{Unbalanced Sinkhorn Divergence}, a differentiable approximation of Wasserstein-2 distance that relaxes mass-matching constraints.

We define $\widehat{\mathbb{P}}_{\boldsymbol{\theta}}$ as the empirical distribution induced by applying the imputer $G_{\boldsymbol{\theta}}$ to the observed dataset, i.e., $\widehat{\mathbb{P}}_{\boldsymbol{\theta}} \coloneqq\sum_{i=1}^N \delta_{\widehat{\boldsymbol{X}}^{(i)}}/N$, where $\widehat{\boldsymbol{X}}^{(i)}$ is the $i$\textsuperscript{th} imputed sample. Then, our imputation objective is minimizing the combination of reconstruction error and the worst-case Unbalanced Sinkhorn Divergence between the imputer distribution and candidate distributions $\mathbb{Q}$ within the ambiguity set.  
Accordingly, we define our primal objective as:
\begin{equation}\label{eq:dri-objective}
    \min_{\boldsymbol{\theta}}\;
        \alpha R_{\boldsymbol{\theta}}
    + (1-\alpha)
    \sup_{\mathbb{Q} \in \mathcal{B}_\rho(\widehat{\mathbb{P}}_N)} S_{\epsilon, \tau}\left(\mathbb{Q}, \, \widehat{\mathbb{P}}_{\boldsymbol{\theta}}\right),
\end{equation}
where 
$R_{\boldsymbol{\theta}}$ represents point-wise reconstruction error, e.g., mean-squared error  normalized by the number of observed entries
$\sum_{i=1}^N \| \boldsymbol{M}^{(i)} \odot ({\boldsymbol{X}}^{(i)} - \widehat{\boldsymbol{X}}^{(i)}) \|_F^2 / \|\boldsymbol{M}^{(i)}\|_0$, where $\|\cdot\|_0$ denotes the matrix zero-norm (counting non-zero elements). Here, $\alpha \in [0,1]$ is the trade-off between point-wise fidelity and the distributional robustness of the imputer.
By simultaneously minimizing the reconstruction error on observed data and the worst-case divergence between the imputer distribution $\widehat{\mathbb{P}}_{\boldsymbol{\theta}}$ and the adversarial $\mathbb{Q}$, the imputer prevents overfitting to the biased empirical distribution from the observed data while maintaining the ability to impute the data point-wise.

\textbf{Computable Reformulation of Imputation Objective.}\label{sec:dual}
Solving \eqref{eq:dri-objective} is computationally intractable because it involves a supremum over an infinite-dimensional space of probability measures $\mathbb{Q}$. Now, we demonstrate that solving it can be reduced to solving a tractable penalized upper-bound representation, which is a minimax problem over \textit{deterministic adversarial sample trajectories}.
\begin{theorem}[Upper Bound for Worst-Case Alignment]
\label{thm:dual-formulation}
    Let $\mathcal{Z} = \{\boldsymbol{\zeta}^{(i)}\}_{i=1}^N \in \mathbb{R}^{N \times D \times T}$ be the batch of adversarial trajectories. Define $\widehat{\mathbb{Q}}_{\mathcal{Z}} \coloneqq \frac{1}{N} \sum_{i=1}^N \delta_{\boldsymbol{\zeta}^{(i)}}$ as the empirical adversary distribution. Then, for any $\gamma \geq 0$ the worst-case distributional alignment is upper-bounded:
    \begin{equation}
            \sup_{\mathbb{Q} \in \mathcal{B}_\rho(\widehat{\mathbb{P}}_N)} S_{\epsilon, \tau}\left(\mathbb{Q}, \, \widehat{\mathbb{P}}_{\boldsymbol{\theta}}\right)
            \leq 
            \gamma \rho + \sup_{\mathcal{Z} \in \mathbb{R}^{N \times D \times T}} \left( S_{\epsilon, \tau}\left(\widehat{\mathbb{Q}}_{\mathcal{Z}}, \widehat{\mathbb{P}}_{\boldsymbol{\theta}}\right)
        - \gamma\, C_{\mathcal{Z}}\right),
    \end{equation}
    where $C_{\mathcal{Z}}:= \sum_{i=1}^N c_{\overline{\boldsymbol{X}}^{(i)}}(\boldsymbol{\zeta}^{(i)}) / N $.
    
\end{theorem}
\begin{proof}[Proof Sketch]
We first rewrite the worst-case problem over distributions $\mathbb Q$
as an equivalent constrained maximization over couplings whose first
marginal is fixed to the empirical measure $\widehat{\mathbb P}_N$.
We then relax the hard transport-budget constraint $\rho$ by introducing
a Lagrange multiplier $\gamma \ge 0$. By weak duality, this yields a
conservative upper bound on the original worst-case alignment problem.

Since the empirical measure $\widehat{\mathbb P}_N$ is discrete, any
feasible coupling can be decomposed into conditional adversarial
distributions associated with each empirical sample. The resulting inner
objective is convex in these conditional distributions because the
Unbalanced Sinkhorn divergence is convex with respect to its input
measure \citep{sejourne2019sinkhorn} and the transport-cost penalty is linear. A Jensen argument then
shows that optimizing over general conditional distributions is
equivalent, in value, to optimizing over deterministic conditional
distributions, i.e., Dirac measures supported at adversarial trajectories
$\{\boldsymbol\zeta^{(i)}\}_{i=1}^N$. Therefore, the intractable search
over infinite-dimensional probability measures reduces to a
finite-dimensional adversarial search over the tensor
$\mathcal Z \in \mathbb R^{N\times D\times T}$.
Moreover, since the weak-duality bound holds for every fixed
$\gamma\ge0$, the corresponding fixed-penalty upper bound also holds for
any chosen $\gamma$. The full proof is provided in
Appendix~\S\ref{sec:theory-proof}.
\end{proof}

\begin{corollary}[Fixed-Penalty DRIO Surrogate]
\label{cor:fixed-gamma-drio}
Fix $\alpha\in[0,1]$, $\rho\ge 0$, $\epsilon>0$, and $\tau>0$.
Following \cref{thm:dual-formulation}, for any fixed $\gamma\ge0$ and
any imputer parameter $\boldsymbol\theta$,
\begin{equation}
\alpha R_{\boldsymbol{\theta}}
+
(1-\alpha)
\sup_{\mathbb{Q}\in\mathcal{B}_\rho(\widehat{\mathbb{P}}_N)}
S_{\epsilon,\tau}
\left(
\mathbb{Q},
\widehat{\mathbb{P}}_{\boldsymbol{\theta}}
\right)
\nonumber
\le
\alpha R_{\boldsymbol{\theta}}
+
(1-\alpha)
\left[
\gamma\rho
+
\sup_{\mathcal Z}
\left\{
S_{\epsilon,\tau}
\left(
\widehat{\mathbb Q}_{\mathcal Z},
\widehat{\mathbb P}_{\boldsymbol\theta}
\right)
-
\gamma C_{\mathcal Z}
\right\}
\right].
\end{equation}
Since $\gamma\rho$ is constant with respect to both
$\boldsymbol\theta$ and $\mathcal Z$, training with this fixed-penalty
upper-bound surrogate is equivalent to minimizing
\begin{equation}\label{eq:final-loss}
    \min_{\boldsymbol{\theta}}
    \;
    \alpha R_{\boldsymbol{\theta}}
    +
    (1-\alpha)
    \sup_{\mathcal{Z}}
    \left\{
    S_{\epsilon,\tau}
    \left(
    \widehat{\mathbb{Q}}_{\mathcal{Z}},
    \widehat{\mathbb{P}}_{\boldsymbol{\theta}}
    \right)
    -
    \gamma C_{\mathcal{Z}}
    \right\},
    \tag{\objname}
\end{equation}
\end{corollary}
\Cref{thm:dual-formulation} provides a conservative upper-bound surrogate for the intractable worst-case alignment term. Specifically, by weak duality, the original supremum over probability measures is upper-bounded by a penalized adversarial objective over empirical adversarial trajectories $\mathcal Z$. The subsequent reduction shows that, after this upper-bound relaxation, the inner search over conditional adversarial distributions can be represented by deterministic adversarial samples, yielding the finite-dimensional objective \eqref{eq:final-loss} in Corollary~\ref{cor:fixed-gamma-drio}.

Importantly, this result does not imply that the fixed-$\gamma$ objective is equivalent to the original DRO objective, nor do we claim that the duality gap is zero. In fact, the exact duality gap is not directly computable here, since it would require solving the original infinite-dimensional primal worst-case problem, which is precisely the intractable problem the upper-bound relaxation is designed to avoid. Instead, DRIO follows a surrogate-optimization principle: we minimize a tractable objective that pointwise upper-bounds the worst-case alignment term. Furthermore, for any fixed $\gamma\ge0$, the additive term $\gamma\rho$ is constant with respect to both $\boldsymbol\theta$ and $\mathcal Z$, and therefore can be omitted from training. We then treat $\gamma$ as a penalty parameter controlling the effective robustness level and select it by validation. 

In \ref{eq:final-loss}, $\alpha$ controls the trade-off between reconstruction error and worst-case distributional alignment and $\gamma$ effectively controls the level of robustness. Larger $\gamma$ increases penalty on the adversary's transport cost, limiting the worst-case distribution to remain close to the empirical data. Conversely, smaller $\gamma$ allows exploring a broader neighborhood of the empirical distribution for imputation. 

\textbf{Training Imputer with $\objname$.} We develop an alternating learning procedure (Algorithm~\ref{alg:dro_training}) to train imputer $G_{\boldsymbol{\theta}}$ with the \ref{eq:final-loss} objective via two updates in each iteration. 
We detail these steps below:

{\em Step 1: Inner Maximization (Adversary Update).} 
In this step, we fix the imputer parameters $\boldsymbol{\theta}$ and search for the worst-case adversarial batch. Each adversary is initialized at the global mean of the observations, i.e., $\boldsymbol{\zeta}^{(i)}_{d,t} = \overline{x}_{d,t}$ for all samples $i$ in the training batch. Then, we update these adversaries by gradient ascent to increase the Sinkhorn divergence $S_{\epsilon, \tau}$, while simultaneously restricting it to the observed data via the ground cost penalty $\gamma C_{\mathcal{Z}}$. To ensure computational efficiency and stability, we fix the imputed data (thus $\widehat{\mathbb{P}}_{\boldsymbol{\theta}})$ and detach it from the computation graph to prevent updates to $\boldsymbol{\theta}$ during this step.

{\em Step 2: Outer Minimization (Imputer Update).} 
In the second step, we fix the discovered adversarial batch and update the imputer parameters $\boldsymbol{\theta}$ by gradient descent to minimize our imputation objective \eqref{eq:final-loss}, with fixed adversaries. The imputer thus learns to reconstruct the observations while forcing the generated trajectories to cover the support of the worst-case data manifold defined by the transport budget and the observed samples.

\renewcommand{\algorithmiccomment}[1]{\hfill{\scriptsize$\triangleright$ #1}}
\begin{algorithm}[t]
   \caption{Training MTS Imputer with $\objname$}
   \label{alg:dro_training}
\begin{algorithmic}[1]
   \STATE {\bfseries Input:} Dataset $\mathcal{D}$, trade-off $\alpha$, robustness $\gamma$, inner steps $K$, batch size $B$, learning rates $\eta_{\zeta}, \eta_{\theta}$
   \STATE {\bfseries Initialize:} Imputer parameters $\boldsymbol{\theta}$ randomly
   \FOR{each batch $\{\boldsymbol{X}_{\mathrm{obs}}^{(i)}, \boldsymbol{M}^{(i)}\}_{i=1}^B \sim \mathcal{D}$}
      \STATE Compute batch mean: $\overline{x}_{d,t} \gets \sum_{i=1}^{B}\boldsymbol{X}^{(i)}_{\mathrm{obs},d,t} / \sum_{j=1}^{B}\boldsymbol{M}^{(j)}_{d,t}$
      \STATE Generate imputation: $\widehat{\boldsymbol{X}}^{(i)} \gets \boldsymbol{X}^{(i)}_{\mathrm{obs}} + (\boldsymbol{1} - \boldsymbol{M}^{(i)}) \odot G_{\boldsymbol{\theta}}(\boldsymbol{X}^{(i)}_{\mathrm{obs}}, \boldsymbol{M}^{(i)})$
      \STATE Initialize adversary: $\mathcal{Z}_{0} \gets \{\boldsymbol{\zeta}^{(i)} : \boldsymbol{\zeta}^{(i)}_{d,t} = \overline{x}_{d,t}\}_{i=1}^B$
      \FOR{$k=1$ {\bfseries to} $K$}
         \STATE $J(\mathcal{Z}_{k-1}) \gets S_{\epsilon, \tau}\big(\widehat{\mathbb{Q}}_{\mathcal{Z}_{k-1}}, \widehat{\mathbb{P}}_{\boldsymbol{\theta}}\big) - \frac{\gamma}{B} \sum_{i=1}^B c_{\overline{\boldsymbol{X}}}(\boldsymbol{\zeta}^{(i)}_{k-1})$
      \COMMENT{Inner maximization with fixed $\boldsymbol{\theta}$}
         \STATE $\mathcal{Z}_{k} \gets \mathcal{Z}_{k-1} + \eta_{\zeta} \nabla_{\mathcal{Z}} J(\mathcal{Z}_{k-1})$
      \ENDFOR
      
      \STATE Re-generate imputation: $\widehat{\boldsymbol{X}}^{(i)} \gets \boldsymbol{X}^{(i)}_{\mathrm{obs}} + (\boldsymbol{1} - \boldsymbol{M}^{(i)}) \odot G_{\boldsymbol{\theta}}(\boldsymbol{X}^{(i)}_{\mathrm{obs}}, \boldsymbol{M}^{(i)})$
      \COMMENT{Outer minimization with fixed adversary trajectories}
        \STATE Compute loss: $\mathcal{L}(\boldsymbol{\theta}) \gets \alpha\,R_{\boldsymbol{\theta}} + (1-\alpha) S_{\epsilon, \tau}\big(\widehat{\mathbb{Q}}_{\mathcal{Z}_{K}}, \widehat{\mathbb{P}}_{\boldsymbol{\theta}}\big)$ 
        \COMMENT{The $-\gamma C_{Z_K}$ term is constant in $\theta$ and omitted in the outer update.}
      \STATE Update: $\boldsymbol{\theta} \gets \boldsymbol{\theta} - \eta_{\theta} \nabla_{\boldsymbol{\theta}}\mathcal{L}(\boldsymbol{\theta})$
   \ENDFOR
\end{algorithmic}
\end{algorithm}

\textbf{Cross-Validation.} Hyperparameter tuning for imputation is challenging because ground-truth missing values are unavailable at deployment; thus, we select hyperparameters using reconstruction error on observed entries from a held-out validation set of unseen samples.
This strategy is deployable as it requires no auxiliary masking and allows learning within-sample dependencies from all observed entries while generalizing across samples. The optimal $(\alpha^*, \gamma^*)$ is selected by minimizing validation loss on these held-out samples. Detailed cross validation procedure and deployability analyses are provided in Appendix~\ref{sec:hyperparams} and \ref{sec:cv_compare}, respectively.

%% file: numerical.tex
\section{Numerical Experiments}
\label{sec:numerical}
We conduct comprehensive experiments to answer three questions:
\textbf{(a)} How does the proposed objective~\eqref{eq:final-loss} perform across diverse datasets and missing mechanisms?
\textbf{(b)} What are the benefits of our design choices?
\textbf{(c)} How effective is the MTS data imputed by $\objname$?
Due to space constraints, detailed experimental settings and extended performance comparison results with sensitivity analyses are provided in Appendix~\ref{sec:experiment-details}. 
Code is available at: \url{https://github.com/CheYiLiao/DRIO}.

\textbf{Datasets.} 
We evaluate on seven commonly used publicly available MTS datasets spanning healthcare, manufacturing, transportation, and environmental monitoring (Table~\ref{tab:datasets}), with sample sizes $N \in [165, 4000]$, features $D \in [3, 78]$, and sequence lengths $T \in [48, 207]$, providing a comprehensive benchmark for assessing imputation performance on distinct real-data scenarios. To evaluate imputers' performances, we introduce artificial missingness for each dataset then partition them into train/validation/test splits (70\%/10\%/20\%). 

\begin{table}[thb]
\centering
\caption{Dataset summary. All datasets are presented as $(N, D, T)$ tensors representing samples $\times$ features $\times$ time steps.}
\label{tab:datasets}
\setlength{\tabcolsep}{1.2pt}
\begin{adjustbox}{width=1\linewidth}
\begin{tabular}{@{}l@{}cccp{11cm}}
\toprule
\textbf{Dataset} & \textbf{N} & \textbf{D} & \textbf{T} & \textbf{Description} \\
\midrule
CNNpred & 165 & 78 & 60 & Quarters (5 US indices combined) $\times$ market indicators $\times$ trading days \\
PEMS08 & 170 & 3 & 62 & Sensors $\times$ (flow, occupancy, speed) $\times$ days \\
PM2.5 & 260 & 7 & 168 & Weeks $\times$ meteorological features $\times$ hours \\
GasSensor & 290 & 16 & 150 & Experiment chunks (58 exp $\times$ 5 chunks) $\times$ sensors $\times$ time steps (60s per chunk) \\
CMAPSS & 359 & 21 & 207 & Engines (with $\geq 207$  cycles from 4 engines) $\times$ sensor measurements $\times$ cycles \\
HAR & 2947 & 9 & 128 & Activity windows $\times$ signals $\times$ time steps \\
PhysioNet & 4000 & 35 & 48 & ICU patients $\times$ clinical variables $\times$ hours \\
\bottomrule
\end{tabular}
\end{adjustbox}
\end{table}

\textbf{Missing Data Mechanisms.} We consider both MCAR and MNAR scenarios under three missing ratios (10\%, 50\%, and 90\%), applied on observed entries. For MCAR, we uniformly mask a fraction of entries independent of their values. For MNAR, we simulate a realistic pattern where extreme values are more likely to be missing \citep{bollinger2014trouble, pham2022missing}. Specifically, each entry $\boldsymbol{X}^{(i)}_{d,t}$ is assigned a missing probability proportional to $\Phi(|z_{d,t}^{(i)}|)$, where $z_{d,t}^{(i)} = (\boldsymbol{X}^{(i)}_{d,t} - \overline{\boldsymbol{x}}_d) / \sigma_d$ is the z-score computed from \textit{feature-wise} mean $(\overline{\boldsymbol{x}}_d)$ and standard deviation $(\sigma_d)$ over the entire dataset, and $\Phi$ is the standard normal CDF. This ensures entry values farther from the feature mean have higher missingness probability.

\textbf{Benchmarks.} 
We compare against two simple baselines and eight representative MTS imputation methods, covering diffusion-based imputers, modern deep learning backbones, and optimal-transport-based methods. 
For baselines, \textbf{Mean} imputes each missing entry with the sample mean at the corresponding feature--time position, i.e., $\bar{\boldsymbol{x}}_{d,t}$; \textbf{MF} (Matrix Factorization) performs low-rank completion after flattening the temporal and feature dimensions. For diffusion-based methods, \textbf{CSDI}~\citep{tashiro2021csdi} uses conditional score-based diffusion models, treating observed values as conditioning information and generating missing values through iterative denoising. \textbf{SSSD}~\citep{alcaraz2022diffusion} combines diffusion-based imputation with structured state-space models to better capture temporal dynamics in time series.
For methods with advanced architecture, \textbf{BRITS}~\citep{cao2018brits} is a bidirectional recurrent imputer that treats missing values as latent variables and combines reconstruction loss with forward--backward consistency. \textbf{SAITS}~\citep{du2023saits} is a self-attention-based imputer designed specifically for time series, using attention blocks to learn temporal dependencies from partially observed sequences. \textbf{ImputeFormer}~\citep{nie2024imputeformer} is a Transformer-based MTS imputer that models high-dimensional temporal dependencies with an attention architecture. \textbf{notMIWAE}~\citep{ipsen2020not} extends variational autoencoders to MNAR settings by jointly modeling the data distribution and the missingness mechanism.
For OT-based methods, \textbf{MDOT}~\citep{muzellec2020missing} minimizes Wasserstein distance for matrix completion; following its tabular-data formulation, we flatten the temporal and feature dimensions. \textbf{PSW}~\citep{wang2025optimal} uses proximal Sinkhorn distances for time-series imputation and is applied to each sample trajectory. 

\textbf{Metrics.}
We employ two complementary metrics on the test set: point-wise reconstruction accuracy and joint distributional alignment. For reconstruction accuracy, we report mean-square error (\textbf{MSE}) computed strictly on \textit{artificially held-out entries} where ground truth is available. MSE measures whether the imputer recovers the missing values accurately at the entry level.
For distributional alignment, we report squared maximum mean discrepancy (\textbf{MMD}$^2$) between the empirical distributions of imputed and ground-truth masked trajectory vectors. MMD$^2$ measures whether the imputed trajectories preserve the joint feature-temporal distribution of the true held-out values. Due to space limits, formal definitions and implementation details are provided in Appendix~\ref{sec:error-metric}.

\paragraph{Results.}
To answer the questions in \S\ref{sec:numerical}, we present results by imputation performance, ablation studies, and downstream task performance.
Note that all data splits are normalized using entry-wise mean and standard deviation from the training set to ensure fair comparison and meaningful metrics reporting.

\textbf{Imputation Performance.}
Table~\ref{tab:wide_mnar} reports results under the more challenging MNAR setting. Overall, $\objname$ achieves the most consistent performance across point-wise accuracy (MSE) and joint feature--temporal distributional alignment (MMD$^2$), ranking among the top three methods for both metrics on all seven datasets. While CSDI, SAITS, ImputeFormer, and PSW perform strongly on selected datasets, their gains are less uniform across metrics and datasets. Moreover, compared with the MCAR setting (\cref{tab:appendix_mcar_mse_mmd_stacked} in Appendix~\ref{sec:more_imputation_performance}), $\objname$ exhibits the smallest performance degradation from MCAR to MNAR, showing the benefit of incorporating distributional uncertainty into the imputation procedure. Together, these results suggest that $\objname$ provides stable and robust reconstruction under distributional uncertainty and systematic missingness. Additional evaluations, including frequency-domain Wasserstein distance and one-dimensional Wasserstein distance, further support $\objname$'s ability to achieve robust distributional alignment for imputation (See  Appendix~\ref{sec:more_imputation_performance}).

\begin{table*}[thb]
\centering
\caption{
Per-dataset MSE and MMD$^2$ under MNAR, averaged across missing ratios.  \textbf{Bold}, \underline{underline}, and \textit{italic} denote the best, second-best, and third-best results, respectively. $\objname$ is paired with the SAITS backbone. All training, validation, and testing samples are normalized using training data mean and variance for each dataset before imputation. Abbreviations: CNN = CNNpred, PeMS = PeMS08, Gas = GasSensor, CMAP = CMAPSS, Physio = PhysioNet, IF = ImputeFormer, and nMW = notMIWAE. Full results with standard deviations are reported in Table~\ref{tab:appendix_mnar_mse_mmd_stacked} in Appendix~\ref{sec:more_imputation_performance}.}
\label{tab:wide_mnar}
\begin{threeparttable}
\begin{adjustbox}{width=\linewidth}
\begin{tabular}{l c c c c c c c |c c c c c c c}
\toprule
& \multicolumn{7}{c}{MSE (point-wise reconstruction, lower better)} 
& \multicolumn{7}{c}{MMD$^2$ (joint distributional alignment, lower better)} \\
\cmidrule(lr){2-8} \cmidrule(lr){9-15}
Method 
& CNN & PeMS & PM2.5 & Gas & CMAP & HAR & Physio
& CNN & PeMS & PM2.5 & Gas & CMAP & HAR & Physio \\
\midrule
\multicolumn{15}{l}{\textit{Baselines}} \\
Mean & 1.350 & 1.203 & 1.489 & 1.402 & 1.155 & 1.234 & 1.351 & 0.208 & 0.255 & 0.231 & 0.295 & 0.201 & 0.184 & 0.177 \\
MF & 1.670 & 2.368 & 1.763 & 0.565 & 1.164 & 1.895 & 2.258 & 0.202 & 0.132 & 0.179 & 0.121 & 0.175 & 0.116 & 0.065 \\
\midrule
\multicolumn{15}{l}{\textit{Benchmarks}} \\
CSDI & 1.232 & 47.850 & 1.078 & 1.164 & \textit{0.136} & \textit{0.287} & \textbf{0.713} & 0.122 & 0.171 & 0.144 & \textbf{0.000} & 0.041 & \textit{0.014} & \textbf{0.021} \\
SSSD & 13.309 & 0.596 & 0.996 & 2.446 & 1.904 & 0.379 & 3.819 & 0.193 & 0.052 & 0.069 & 0.058 & 0.025 & 0.048 & 0.051 \\
BRITS & \textit{0.764} & 0.544 & 0.913 & 0.085 & 0.360 & 0.457 & 0.868 & 0.123 & 0.161 & 0.145 & 0.009 & 0.087 & 0.143 & 0.137 \\
SAITS & \underline{0.680} & \textbf{0.250} & \textit{0.874} & 0.030 & \underline{0.104} & \underline{0.260} & 0.775 & \textit{0.102} & \underline{0.015} & \underline{0.050} & 0.002 & \underline{0.006} & \underline{0.005} & \underline{0.026} \\
IF & 0.884 & \textit{0.288} & 1.008 & \textbf{0.012} & 0.160 & 0.424 & \textit{0.719} & 0.104 & \textit{0.015} & 0.068 & \textit{0.000} & \textit{0.012} & 0.023 & 0.049 \\
nMW & 2.232 & 2.828 & 2.999 & 1.280 & 0.648 & 2.155 & 7.830 & 0.291 & 0.353 & 0.331 & 0.268 & 0.153 & 0.315 & 0.230 \\
MDOT & 0.979 & 0.652 & 1.167 & 0.411 & 0.789 & 0.772 & 1.131 & 0.179 & 0.179 & 0.195 & 0.164 & 0.145 & 0.146 & 0.164 \\
PSW & 1.313 & 0.376 & \textbf{0.422} & \underline{0.016} & 0.925 & 0.419 & \underline{0.716} & \textbf{0.051} & 0.033 & \textit{0.062} & \underline{0.000} & 0.047 & 0.045 & 0.029 \\
\midrule
\multicolumn{15}{l}{\textit{Ours}} \\
DRIO & \textbf{0.645} & \underline{0.252} & \underline{0.818} & \textit{0.024} & \textbf{0.097} & \textbf{0.256} & 0.773 & \underline{0.087} & \textbf{0.015} & \textbf{0.039} & 0.001 & \textbf{0.004} & \textbf{0.005} & \textit{0.026} \\
\bottomrule
\end{tabular}
\end{adjustbox}
\end{threeparttable}
\end{table*}

\textbf{Ablation Study.}
Table~\ref{tab:ablation} separates the effect of architecture and training objective. With the $\objname$ objective fixed, SAITS achieves the best performance across both MCAR and MNAR, substantially outperforming standard MLP (Multi-Layer Perceptron), LSTM (Long Short-Term Memory) and STT (Spatiotemporal Transformer) that are not designed specifically for time series imputation. As SAITS is also attention-based, the results suggest that while $\objname$ can be paired with various deep learning architectures, it generally benefits from architectures designed for time series imputation.
With the SAITS architecture fixed, $\objname$ consistently improves over pure MSE, DRIO with balanced Sinkhorn (B-DRIO), and the original SAITS procedure with internal masking. We note that $\objname$ acts as a regularizer preventing overfitting to observed entries, which shares similar idea with internal masking used in original SAITS imputer. That said, $\objname$ does so through explicit worst-case distributional uncertainty rather than random masking. Additionally, the improvement over B-DRIO further supports the use of unbalanced Sinkhorn divergence, especially under MNAR, as strict mass conservation can be too restrictive under support mismatch. Overall, these results indicate that $\objname$'s gains come from both a suitable imputation backbone and its distributionally robust objective.

\begin{table}[t]
\centering
\caption{Ablation on (i) imputation architecture with the DRIO objective fixed (left), and (ii) training objective with the SAITS architecture fixed (right). Each cell reports mean (std) across datasets and missing ratios. B-DRIO is \ref{eq:final-loss} with the balanced Sinkhorn variant. MAE+Internal Masking is the original SAITS procedure (MAE from both observed entries and  additional internal masks).}
\label{tab:ablation}
\begin{adjustbox}{width=\linewidth}
\begin{tabular}{l l c c c c c c c}
\toprule
& & \multicolumn{4}{c}{\textit{Architecture (DRIO obj.)}} & \multicolumn{3}{c}{\textit{Objective (SAITS arch.)}} \\
\cmidrule(lr){3-6} \cmidrule(lr){7-9}
Mech. & Metric & MLP & LSTM & STT & SAITS & B-DRIO & MSE & MAE+Internal Masking \\
&  &  &  &  & &  & & (SAITS-orig.) \\
\midrule
\multirow{2}{*}{MCAR} & MSE & 0.835 {\tiny (0.549)} & 0.614 {\tiny (0.452)} & 8.171 {\tiny (17.565)} & \textbf{0.326} {\tiny (0.368)} & 0.361 {\tiny (0.379)} & 1.271 {\tiny (0.332)} & 0.343 {\tiny (0.382)} \\
 & MMD$^2$ & 0.148 {\tiny (0.111)} & 0.106 {\tiny (0.107)} & 0.275 {\tiny (0.196)} & \textbf{0.020} {\tiny (0.029)} & 0.041 {\tiny (0.067)} & 0.242 {\tiny (0.095)} & 0.024 {\tiny (0.035)} \\
\midrule
\multirow{2}{*}{MNAR} & MSE & 0.876 {\tiny (0.511)} & 0.687 {\tiny (0.459)} & 23.181 {\tiny (72.544)} & \textbf{0.367} {\tiny (0.349)} & 0.469 {\tiny (0.421)} & 1.463 {\tiny (0.322)} & 0.388 {\tiny (0.375)} \\
 & MMD$^2$ & 0.161 {\tiny (0.118)} & 0.126 {\tiny (0.117)} & 0.282 {\tiny (0.210)} & \textbf{0.029} {\tiny (0.046)} & 0.078 {\tiny (0.105)} & 0.253 {\tiny (0.086)} & 0.035 {\tiny (0.056)} \\
\bottomrule
\end{tabular}
\end{adjustbox}
\end{table}

\textbf{Downstream Tasks.}
Figure~\ref{fig:downstream} summarizes downstream forecast MSE using the imputed time series; details are provided in Appendix~\ref{sec:downstream_task}. Methods are ordered by mean MSE (marked with red bars), while boxplots show the median and interquartile, computed across all datasets and missing ratios.

Overall, $\objname$ achieves the lowest mean downstream MSE under both MCAR and MNAR. Under MCAR, $\objname$ obtains mean MSE $0.486$, improving over the strongest non-$\objname$ baseline, PSW ($0.519$). Under MNAR, $\objname$ obtains mean MSE $0.550$, slightly improving over ImputeFormer ($0.563$). Although the gains are more modest than the direct imputation improvements, the consistent ranking across missing mechanisms suggest that $\objname$'s benefit also transfers to downstream predictive utility. Moreover, the boxplots indicate that $\objname$ maintains a tighter distribution across scenarios, compared to other benchmarks. These results suggest that distributionally robust imputation indeed preserve predictive temporal structure under distributional uncertainty and systematic missingness.

\begin{figure}[thb]
\begin{center}
\centerline{\includegraphics[width=0.9\columnwidth]{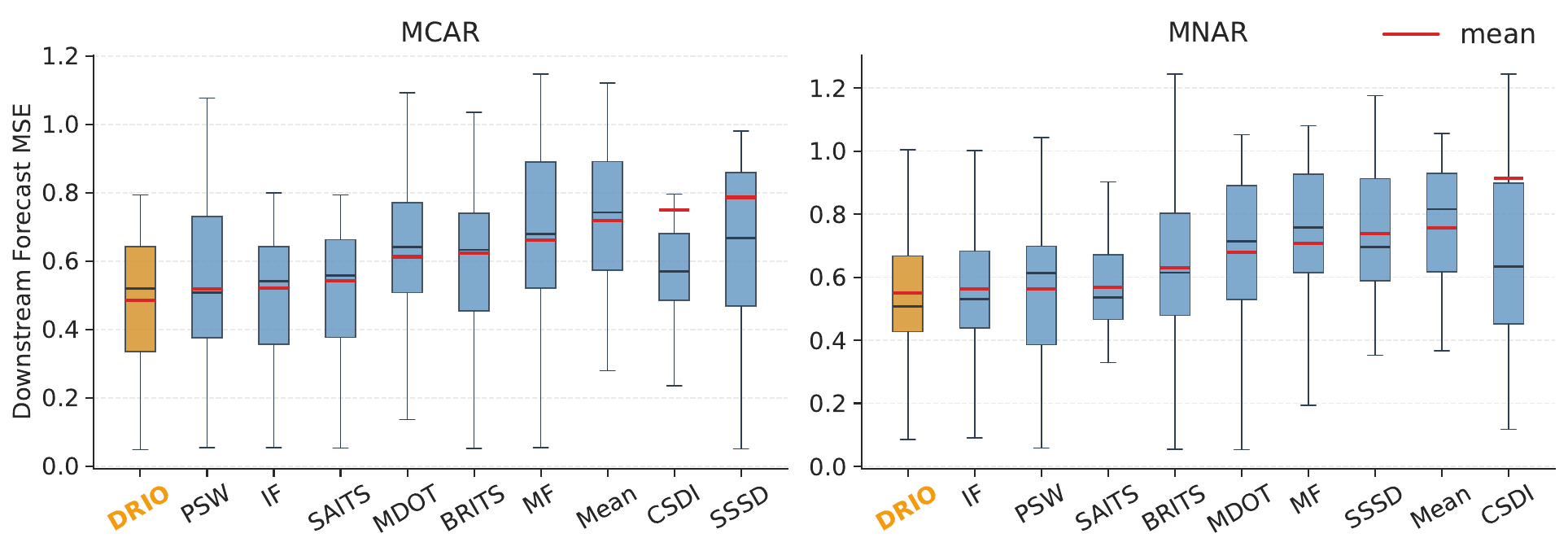}}
\caption{Downstream forecasting MSE using imputed data. Each box aggregates scenario-level forecast MSE across datasets and missing ratios under the corresponding missing mechanism; red bars denote means. Methods are ordered by mean forecast MSE (red bars).}
\label{fig:downstream}
\end{center}
\end{figure}

%% file: conclusion.tex
\section{Limitations}\label{sec:limitations}
Our work has two main limitations. First, our tractable formulation minimizes a weak-duality upper bound of the worst-case objective, and we do not guarantee that the duality gap is closed. However, controlling this conservative upper bound still hedges against the original worst-case risk, and cross-validated results show consistent gains in both imputation and downstream forecasting across missingness settings. Second, $\objname$ requires alternating adversarial updates with Sinkhorn computation, which increases \textit{training} cost relative to standard reconstruction-only objectives. Nevertheless, this overhead is practically manageable in our experiments, with most training runs completed within minutes (see Appendix~\ref{sec:run_time}), and inference cost is unchanged because adversarial updates are not used after training.

\section{Conclusion}
We propose $\objname$, a distributionally robust regularized objective for multivariate time series imputation under non-stationarity and systematic missingness. By hedging against worst-case distributional shifts within a Wasserstein ambiguity set, $\objname$ complements point-wise reconstruction with robust distributional regularization. We derive a tractable adversarial training formulation and show that it can be paired with modern imputation backbones. Experiments across diverse datasets demonstrate that $\objname$ improves robustness across missingness scenarios and suggests improved downstream forecasting utility. These results suggest that explicitly modeling distributional uncertainty is a promising direction for reliable time series imputation.

%% file: appendix.tex
\newpage
\appendix
\section{Loss Function Construction} \label{sec:discuss-loss}
We provide a comprehensive discussion on the Sinkhorn divergence used in our formulation \eqref{eq:dri-objective}.

Let $\mathcal{Z} = \mathbb{R}^{D \times T}$ denote the feature-temporal space of the observed data dimension. For any two probability measures $\mu, \nu \in \mathcal{P}(\mathcal{Z})$, the entropic unbalanced transport cost $W_{\epsilon, \tau}$ is defined as the solution to the minimization problem over all positive transport plans $\pi$:
\begin{equation}
    W_{\epsilon, \tau}(\mu, \nu) \coloneqq \inf_{\pi \ge 0}\left\{ \int_{\mathcal{Z}^2} d(\boldsymbol{x}, \boldsymbol{z}) \, \mathrm{d}\pi(\boldsymbol{x}, \boldsymbol{z})+ \epsilon \text{KL}(\pi | \mu \otimes \nu) + \tau ( \text{KL}(\pi_1 | \mu) + \text{KL}(\pi_2 | \nu) ) \right \},
\end{equation}
where $\pi$ is a positive Radon measure on $\mathcal{Z} \times \mathcal{Z}$ with marginals $\pi_1$ and $\pi_2$, $d(\boldsymbol{x}, \boldsymbol{z}) = \|\boldsymbol{x} - \boldsymbol{z}\|_F^2$ is the ground metric, and $\text{KL}(\cdot|\cdot)$ denotes the Kullback-Leibler divergence. Here, $\epsilon > 0$ is the entropic regularization coefficient that smooths the transport plan, and $\tau > 0$ is the marginal relaxation parameter that allows creating and destroying local mass.

The \textbf{Unbalanced Sinkhorn Divergence} is then defined via the debiasing formula, i.e., 
\begin{equation}
    S_{\epsilon, \tau}(\mu, \nu) \coloneqq W_{\epsilon, \tau}(\mu, \nu) - \frac{1}{2} \left( W_{\epsilon, \tau}(\mu, \mu) + W_{\epsilon, \tau}(\nu, \nu) \right).
\end{equation}
Note that the \textbf{Balanced Sinkhorn Divergence} is recovered as $\tau \to \infty$, which enforces strict marginal constraints. In this case, the transport plan $\pi$ must belong to the set of couplings $\Pi(\mu, \nu) = \{\pi \geq 0 : \pi_1 = \mu, \pi_2 = \nu\}$, and the entropic transport cost simplifies to:
\begin{equation}
    W_{\epsilon}(\mu, \nu) \coloneqq \inf_{\pi \in \Pi(\mu, \nu)} \left\{ \int_{\mathcal{Z}^2} d(\boldsymbol{x}, \boldsymbol{z}) \, \mathrm{d}\pi(\boldsymbol{x}, \boldsymbol{z}) + \epsilon \, \text{KL}(\pi \mid \mu \otimes \nu) \right\}.
\end{equation}
The balanced Sinkhorn divergence is then defined as:
\begin{equation}
S_{\epsilon}(\mu, \nu) \coloneqq W_{\epsilon}(\mu, \nu) - \frac{1}{2} ( W_{\epsilon}(\mu, \mu) + W_{\epsilon}(\nu, \nu) )
\end{equation}
In our ablation study and sensitivity analysis,  analyze the effect of marginal relaxation and confirmed its benefit in MTS imputation (See \cref{sec:numerical} and Appendix~\ref{sec:tau_sensitivity}).

Now, we expand our discussions on the design choices with respect to the unbalanced Sinkhorn divergence used in this work:

\paragraph{Why Wasserstein-based Divergence?}
We employ an Optimal Transport (Wasserstein-based) loss rather than classical divergences because standard metrics do not fully capture the geometry of distributional support. Point-wise metrics such as MSE focus on entry-wise fidelity and can blur predictions in high-uncertainty regions, while likelihood-based divergences such as KL may suffer from vanishing or ill-defined gradients when the model and data distributions have weakly overlapping supports~\citep{wang2025optimal,villani2008optimal}. In contrast, OT compares distributions through the cost of transporting mass, providing meaningful learning signals even when supports are disjoint.

\paragraph{Why Sinkhorn Divergence?}
While Wasserstein geometry is desirable, computing the exact Wasserstein distance requires solving a network-flow linear program with super-cubic complexity, e.g., $O(n^3\log n)$, whose optimal value is generally non-smooth with respect to the imputed samples $\widehat{\boldsymbol{X}}^{(i)}$ and therefore with respect to the imputer parameters $\boldsymbol{\theta}$. This makes exact Wasserstein distance difficult to use inside the inner loop of a deep learning framework. We therefore use Sinkhorn divergence as a smooth Wasserstein-based surrogate. With entropic regularization, the transport problem becomes strongly convex in the transport plan and can be solved efficiently by Sinkhorn iterations with $O(n^2)$ cost per iteration~\citep{villani2008optimal,wang2025sinkhorn}. Crucially, the resulting objective is differentiable and GPU-parallelizable, allowing stable backpropagation to the imputer parameters $\boldsymbol{\theta}$.

In principle, one could work with the exact Wasserstein distance and obtain gradients or subgradients through Lagrangian, adjoint-sensitivity, or envelope-theorem arguments by differentiating through the optimal plan or dual potentials \citep{border2015miscellaneous}. However, this does not remove the minimax structure of our distributionally robust formulation in \cref{eq:dri-objective}, which still necessitate our alternating update in Algorithm~\ref{alg:dro_training}. Moreover, because exact Wasserstein distance is the value of a linear program, the optimizer may be non-unique and the active transport plan can change abruptly as $\boldsymbol{\theta}$ varies. The resulting derivative is therefore generally a subgradient of a non-smooth value function, which can be unstable and costly when repeatedly used in the inner loop.

For these reasons, Sinkhorn divergence provides a more practical choice for our setting: it preserves the geometric alignment benefits of OT while yielding smooth, stable, and scalable gradients. We further use the unbalanced version to relax strict mass conservation, which is important when systematic missingness induces support mismatch between the observed empirical distribution and the true data-generating process.

\paragraph{Why Unbalanced Sinkhorn Divergence?}
We employ the \textit{unbalanced} formulation specifically to mitigate the systematic support bias caused by the combined effects of time series non-stationarity and systematic missing mechanisms.
Standard Wasserstein distances (and balanced Sinkhorn) enforce strict conservation of mass, requiring the integral of the imputed distribution to exactly match that of the empirical reference. When the observed data has truncated support (e.g., missing tails), this constraint distort the probability mass from the reliable observed to fill the unobserved voids.
In contrast, the Unbalanced Sinkhorn Divergence relaxes the marginal constraints via the parameter $\tau$. This allows the transport plan to create or destroy mass locally, which is critical for the designed non-stationary environments, as it allows the imputer $G_{\boldsymbol{\theta}}$ to instantaneously create new modes, with a penalty, without strictly following the global mass balance constraint \cite{sejourne2019sinkhorn, choi2023generative}.

\section{Theoretical Result}\label{sec:theory-proof}
We provide detailed proofs for Theorem~\ref{thm:dual-formulation} and Corollary~\ref{cor:fixed-gamma-drio} in this section.
\begin{proof}[Proof of \cref{thm:dual-formulation}]\hfill

By the definition of the ambiguity set \eqref{eq:ambiguity-set}, the worst-case divergence problem can be rewritten as:
\begin{equation}
\begin{aligned}
\label{eq:rewrite-primal-1}
    \sup_{\mathbb{Q} \in \mathcal{P}(\mathbb{R}^{D \times T})} 
    & S_{\epsilon, \tau}\left(\mathbb{Q}, \widehat{\mathbb{P}}_{\boldsymbol{\theta}}\right) \\ 
    \text{subject to}\;\quad\;
    & \inf_{\pi \in \Pi(\widehat{\mathbb{P}}_N, \mathbb{Q})} \mathbb{E}_{(\boldsymbol{X},\boldsymbol{Z}) \sim \pi} \left[c_{\boldsymbol{X}}\left( \boldsymbol{Z} \right) \right] \leq \rho,
\end{aligned}
\end{equation}
where $\rho \geq 0$.
To derive the dual formulation, we aim to establish that the nested maximization over distributions $\mathbb{Q}$ and the inner minimization over couplings $\pi$ can be unified into a single constrained maximization. Specifically, we first define the maximization problem over the coupling $\pi$:
\begin{equation}
\begin{aligned}\label{eq:rewrite-primal-2}
    \sup_{\substack{\pi \in \mathcal{P}(\mathbb{R}^{D \times T} \times \mathbb{R}^{D \times T}) }} 
    &S_{\epsilon, \tau}\left(\pi_{\boldsymbol{Z}},\widehat{\mathbb{P}}_{\boldsymbol{\theta}}\right) 
    \\
    \text{subject to}\;\quad\;
    &\mathbb{E}_{(\boldsymbol{X},\boldsymbol{Z}) \sim \pi} \left[c_{\boldsymbol{X}}\left( \boldsymbol{Z} \right) \right] \leq \rho\\
    & \pi_{\boldsymbol{X}} = \widehat{\mathbb{P}}_N
\end{aligned}
\end{equation}
where $\pi_{\boldsymbol{X}}$ and $\pi_{\boldsymbol{Z}}$ represent the first and the second marginal of the coupling $\pi$, respectively. Now, we show that  \ref{eq:rewrite-primal-1} is equivalent to \ref{eq:rewrite-primal-2}.

Let $V_1$ denote the optimal value of the original problem \eqref{eq:rewrite-primal-1} and $V_2$ denote the optimal value of the reformulated problem \eqref{eq:rewrite-primal-2}.
First, we start with problem \eqref{eq:rewrite-primal-2}. Consider a feasible coupling $\pi$ that satisfies the marginal constraint $\pi_{\boldsymbol{X}} = \widehat{\mathbb{P}}_N$ and the budget constraint. Then, the following relationship holds by definition:
\begin{equation}
\rho
\geq
\mathbb{E}_{(\boldsymbol{X},\boldsymbol{Z}) \sim \pi} \left[c_{\boldsymbol{X}}\left( \boldsymbol{Z} \right) \right] 
\geq
\inf_{\pi \in \Pi(\widehat{\mathbb{P}}_N, \mathbb{Q})} \mathbb{E}_{(\boldsymbol{X},\boldsymbol{Z}) \sim \pi} \left[c_{\boldsymbol{X}}\left( \boldsymbol{Z} \right) \right],
\end{equation}
where $\mathbb{Q} \in \mathcal{P}(\mathbb{R}^{D\times T})$.
Thus, any solution to \eqref{eq:rewrite-primal-2} is feasible in 
\eqref{eq:rewrite-primal-1}, which implies that $V_2 \leq V_1$.

Conversely, consider problem \eqref{eq:rewrite-primal-1}. By definition, the minimum transport cost from $\widehat{\mathbb{P}}_N$ to any distribution $\mathbb{Q}$ in the ambiguity set is at most $\rho$. Let $\pi^*$ be a coupling representing the optimal transport plan that achieves this minimum cost. Then $\pi^*$ satisfies $\mathbb{E}_{(\boldsymbol{X},\boldsymbol{Z}) \sim \pi^*} \left[c_{\boldsymbol{X}}\left( \boldsymbol{Z} \right) \right]  \le \rho$ with the first marginal $\pi_{\boldsymbol{X} 
}=\widehat{\mathbb{P}}_N$. Consequently, $\pi^*$ is a feasible solution to \eqref{eq:rewrite-primal-2}, implying $V_1 \leq V_2$. This establishes the equivalence between \eqref{eq:rewrite-primal-1} and \eqref{eq:rewrite-primal-2}.

Now, we introduce a Lagrange multiplier $\gamma \geq 0$ to \eqref{eq:rewrite-primal-2} to relax the transport budget constraint, which yields:
\begin{equation}
\label{eq:rewrite-primal-3}
    \sup_{\pi: \pi_{\boldsymbol{X}} = \widehat{\mathbb{P}}_N} \inf_{\gamma \ge 0} \Bigg\{ S_{\epsilon, \tau}\left(\pi_{\boldsymbol{Z}}, \widehat{\mathbb{P}}_{\boldsymbol{\theta}}\right) + \gamma \rho - 
    \gamma \mathbb{E}_{(\boldsymbol{X},\boldsymbol{Z}) \sim \pi} \left[c_{\boldsymbol{X}}\left( \boldsymbol{Z} \right) \right] \Bigg\}.
\end{equation}
By rearranging the terms inside the objective and exchanging the order of supremum and infimum in \eqref{eq:rewrite-primal-3}, we obtain the dual formulation to the primal worst-case divergence problem:
\begin{equation}
\label{eqn:dual-swapped}
    \inf_{\gamma \ge 0} \left\{ \gamma \rho + \sup_{\pi: \pi_{\boldsymbol{X}} = \widehat{\mathbb{P}}_N} \left( S_{\epsilon, \tau}\left(\pi_{\boldsymbol{Z}}, \widehat{\mathbb{P}}_{\boldsymbol{\theta}}\right) - \gamma \mathbb{E}_{(\boldsymbol{X},\boldsymbol{Z}) \sim \pi} \left[c_{\boldsymbol{X}}\left( \boldsymbol{Z} \right) \right] \right) \right\}.
\end{equation}
By weak duality, the formulation \eqref{eqn:dual-swapped} is an upper bound on the primal objective:
\begin{equation}\label{eq:upbound}
                \sup_{\mathbb{Q} \in \mathcal{B}_\rho(\widehat{\mathbb{P}}_N)} S_{\epsilon, \tau}\left(\mathbb{Q}, \, \widehat{\mathbb{P}}_{\boldsymbol{\theta}}\right)
                \leq
    \inf_{\gamma \ge 0} \left\{ \gamma \rho + \sup_{\pi: \pi_{\boldsymbol{X}} = \widehat{\mathbb{P}}_N} \left( S_{\epsilon, \tau}\left(\pi_{\boldsymbol{Z}}, \widehat{\mathbb{P}}_{\boldsymbol{\theta}}\right) - \gamma \mathbb{E}_{(\boldsymbol{X},\boldsymbol{Z}) \sim \pi} \left[c_{\boldsymbol{X}}\left( \boldsymbol{Z} \right) \right] \right) \right\}
\end{equation}

We note that the Unbalanced Sinkhorn Divergence $S_{\epsilon, \tau}(\cdot, \widehat{\mathbb{P}}_{\boldsymbol{\theta}})$ is a convex functional with respect to the input measure. Consequently, the standard Minimax Theorem (which requires a concave-convex structure) does not guarantee strict equality for this swap of the infimum and supremum.
However, establishing \eqref{eqn:dual-swapped} as an upper bound is sufficient for our distributionally robust framework. From theoretical perspective, minimizing this dual objective guarantees that we minimize a conservative surrogate of the worst-case risk. 
Specifically, if the model is robust against this dual upper bound, it is by definition robust against the primal worst-case scenario. 
We now proceed to simplify the inner maximization of this upper bound. Although the outer duality gap exists, we show next that the inner search for the worst-case measure is exact.

\eqref{eqn:dual-swapped}. 
Note that in the inner maximization, the first marginal of $\pi$ is fixed to the discrete empirical distribution $\widehat{\mathbb{P}}_N = \frac{1}{N}\sum \delta_{\overline{\boldsymbol{X}}^{(i)}}$. Therefore, by the disintegration formula, any feasible coupling $\pi$ can be uniquely decomposed into a mixture of $N$ conditional distributions. That is, for any feasible $\pi$ in \eqref{eqn:dual-swapped}, we have
\begin{equation}\label{eq:pi-rewrite1}
\pi
=
\int \delta_{\boldsymbol{X}}\otimes \pi(\cdot|\boldsymbol{X}) \mathrm{d}\widehat{\mathbb{P}}_{N}(\boldsymbol{X})
=
\frac{1}{N} \sum_{i=1}^N \delta_{\overline{\boldsymbol{X}}^{(i)}}\otimes \pi^{(i)}(\boldsymbol{Z}),
\end{equation}
where $\pi^{(i)}(\boldsymbol{Z})$ denotes the conditional distribution of the adversary given sample $\overline{\boldsymbol{X}}^{(i)}$. Then, we have:
$$    
\mathrm{d}\pi(\boldsymbol{X}, \boldsymbol{Z}) = \frac{1}{N} \sum_{i=1}^N \delta_{\overline{\boldsymbol{X}}^{(i)}} \mathrm{d}\boldsymbol{X}\,\otimes \mathrm{d}\pi^{(i)}(\boldsymbol{Z}).
$$    
Consequently, the expected ground cost over the coupling $\pi$ in \eqref{eqn:dual-swapped} can be simplified to an expectation conditional on the samples. That is,
\begin{align}
\mathbb{E}_{(\boldsymbol{X},\boldsymbol{Z}) \sim \pi} \left[c_{\boldsymbol{X}}\left( \boldsymbol{Z} \right) \right]
&= 
\int c_{\boldsymbol{X}}(\boldsymbol{Z}) \, \mathrm{d}\pi(\boldsymbol{X}, \boldsymbol{Z})\\
&= 
\int c_{\boldsymbol{X}}(\boldsymbol{Z}) \, \left( \frac{1}{N} \sum_{i=1}^N \delta_{\overline{\boldsymbol{X}}^{(i)}}(\mathrm{d} \boldsymbol{X}) \otimes \mathrm{d}\pi^{(i)}(\boldsymbol{Z}) \right)\\
&=
\frac{1}{N} \sum_{i=1}^N \int_{\boldsymbol{Z}} \left( \int_{\boldsymbol{X}} c_{\boldsymbol{X}}(\boldsymbol{Z}) \, \delta_{\overline{\boldsymbol{X}}^{(i)}}(\mathrm{d}\boldsymbol{X})\right) \mathrm{d}\pi^{(i)}(\boldsymbol{Z})\\
&=
\frac{1}{N} \sum_{i=1}^N \int_{\boldsymbol{Z}} c_{\overline{\boldsymbol{X}}^{(i)}}(\boldsymbol{Z})
\;
\mathrm{d}\pi^{(i)}(\boldsymbol{Z}).\label{eq:rewrite-cost1}
\end{align}
Next, for the Sinkhorn term in \eqref{eqn:dual-swapped}, we  further rewrite the second marginal $\pi_{\boldsymbol{Z}}$ of the coupling $\pi$ by marginalization using results from \cref{eq:pi-rewrite1}:
\begin{align}
\pi_{\boldsymbol{Z}}
= 
\int \pi\; \mathrm{d}\boldsymbol{X}
&= 
\int \frac{1}{N} \sum_{i=1}^N \delta_{\overline{\boldsymbol{X}}^{(i)}}\otimes \pi^{(i)}(\boldsymbol{Z})\;\mathrm{d}\boldsymbol{X}\\
&= 
\frac{1}{N} \sum_{i=1}^N \int \delta_{\overline{\boldsymbol{X}}^{(i)}}\otimes \pi^{(i)}(\boldsymbol{Z})\;\mathrm{d}\boldsymbol{X}\\
&= \frac{1}{N} \sum_{i=1}^N \pi^{(i)}.\label{eq:rewrite-sinkhorn1}
\end{align}

Therefore, inserting \cref{eq:rewrite-cost1} and \cref{eq:rewrite-sinkhorn1} into \eqref{eqn:dual-swapped}, we know that maximizing over the joint coupling $\pi$ is equivalent to maximizing over the set of conditional distributions $\{\pi^{(i)}\}_{i=1}^N$, which implies that \eqref{eqn:dual-swapped} is equivalent to the following optimization problem:
\begin{equation}
    \inf_{\gamma \ge 0} \left\{ \gamma \rho + \sup_{\{\pi^{(i)}\}_{i=1}^N} \left( S_{\epsilon, \tau}\left(\frac{1}{N}\sum_{i=1}^N \pi^{(i)}, \widehat{\mathbb{P}}_{\boldsymbol{\theta}}\right) - \frac{\gamma}{N} \sum_{i=1}^N \int c_{\overline{\boldsymbol{X}}^{(i)}}(\boldsymbol{Z}) \, \mathrm{d}\pi^{(i)}(\boldsymbol{Z}) \right) \right\}.
\label{eqn:dual-swapped2}
\end{equation}
Next, we show that the search over conditional distributions
$\{\pi^{(i)}\}_{i=1}^N$ can be reduced to a search over deterministic
adversarial trajectories $\{\boldsymbol\zeta^{(i)}\}_{i=1}^N$. For fixed $\gamma\ge 0$, we define
\[
H_\gamma(\pi^{(1)},\ldots,\pi^{(N)})
:=
S_{\epsilon,\tau}
\left(
\frac1N\sum_{i=1}^N \pi^{(i)},
\widehat{\mathbb P}_{\boldsymbol\theta}
\right)
-
\frac{\gamma}{N}
\sum_{i=1}^N
\int
c_{\overline{\boldsymbol X}^{(i)}}(\boldsymbol Z)
\,d\pi^{(i)}(\boldsymbol Z).
\]

Importantly, by the convexity of the unbalanced Sinkhorn divergence in each input measure
under KL marginal relaxation \citep{sejourne2019sinkhorn}, 
$S_{\epsilon,\tau}(\cdot,\widehat P_\theta)$ is convex on the space of positive finite measures.
Since the map $(\pi^{(1)},\ldots,\pi^{(N)}) \mapsto \frac1N\sum_i \pi^{(i)}$ is linear and the transport-cost term is linear, 
$H_\gamma$ is convex in $(\pi^{(1)},\ldots,\pi^{(N)})$.

Now, we further let $\boldsymbol Z_i\sim \pi^{(i)}$ independently and re-parametrize $\pi^{(i)}$ with Dirac measures:
\[
(\pi^{(1)},\ldots,\pi^{(N)})
=
\mathbb E
\left[
(\delta_{\boldsymbol Z_1},\ldots,\delta_{\boldsymbol Z_N})
\right].
\]
With these definitions, we obtain the following by the Jensen's inequality:
\begin{align*}
    H_\gamma(\pi^{(1)},\ldots,\pi^{(N)})
    &=
    H_\gamma( \mathbb E
    \left[
    (\delta_{\boldsymbol Z_1},\ldots,\delta_{\boldsymbol Z_N})
    \right]
    )\\
    &\leq
\mathbb E
\left[
H_\gamma
(\delta_{\boldsymbol Z_1},\ldots,\delta_{\boldsymbol Z_N})
\right]
\le
\sup_{\boldsymbol\zeta^{(1)},\ldots,\boldsymbol\zeta^{(N)}}
H_\gamma
(
\delta_{\boldsymbol\zeta^{(1)}},
\ldots,
\delta_{\boldsymbol\zeta^{(N)}}
).
\end{align*}

The first inequality follows from Jensen's inequality because
$H_\gamma$ is convex in
$(\pi^{(1)},\ldots,\pi^{(N)})$. The second inequality follows because
the expectation of any random variable is upper bounded by its supremum.
Since the above bound holds for arbitrary conditional distributions
$\{\pi^{(i)}\}_{i=1}^N$, taking the supremum over all such conditional
distributions gives
\[
\sup_{\{\pi^{(i)}\}_{i=1}^N}
H_\gamma(\pi^{(1)},\ldots,\pi^{(N)})
\le
\sup_{\boldsymbol\zeta^{(1)},\ldots,\boldsymbol\zeta^{(N)}}
H_\gamma
(
\delta_{\boldsymbol\zeta^{(1)}},
\ldots,
\delta_{\boldsymbol\zeta^{(N)}}
).
\]
Conversely, the class of Dirac measures is a subset of all probability
measures. Hence, for any deterministic tuple
$(\boldsymbol\zeta^{(1)},\ldots,\boldsymbol\zeta^{(N)})$, the tuple
$(\delta_{\boldsymbol\zeta^{(1)}},\ldots,
\delta_{\boldsymbol\zeta^{(N)}})$ is feasible for the optimization over
$\{\pi^{(i)}\}_{i=1}^N$. Therefore,
\[
\sup_{\{\pi^{(i)}\}_{i=1}^N}
H_\gamma(\pi^{(1)},\ldots,\pi^{(N)})
\ge
\sup_{\boldsymbol\zeta^{(1)},\ldots,\boldsymbol\zeta^{(N)}}
H_\gamma
(
\delta_{\boldsymbol\zeta^{(1)}},
\ldots,
\delta_{\boldsymbol\zeta^{(N)}}
).
\]
Combining above two inequalities yields
\[
\sup_{\{\pi^{(i)}\}_{i=1}^N}
H_\gamma(\pi^{(1)},\ldots,\pi^{(N)})
=
\sup_{\mathcal Z\in\mathbb R^{N\times D\times T}}
\left\{
S_{\epsilon,\tau}
\left(
\widehat{\mathbb Q}_{\mathcal Z},
\widehat{\mathbb P}_{\boldsymbol\theta}
\right)
-
\gamma C_{\mathcal Z}
\right\},
\]
where $ 
\mathcal Z=\{\boldsymbol\zeta^{(i)}\}_{i=1}^N$,
$\widehat{\mathbb Q}_{\mathcal Z}
:=
\frac1N\sum_{i=1}^N
\delta_{\boldsymbol\zeta^{(i)}}$, and
$
C_{\mathcal Z}
:=
\frac1N
\sum_{i=1}^N
c_{\overline{\boldsymbol X}^{(i)}}
(\boldsymbol\zeta^{(i)})$.

Substituting this identity into \eqref{eqn:dual-swapped2}, we can rewrite \cref{eq:upbound} as 
\[
            \sup_{\mathbb{Q} \in \mathcal{B}_\rho(\widehat{\mathbb{P}}_N)} S_{\epsilon, \tau}\left(\mathbb{Q}, \, \widehat{\mathbb{P}}_{\boldsymbol{\theta}}\right)
            \leq
\inf_{\gamma\ge0}
\left\{
\gamma\rho+
\sup_{\mathcal Z\in\mathbb R^{N\times D\times T}}
\left[
S_{\epsilon,\tau}
\left(
\widehat{\mathbb Q}_{\mathcal Z},
\widehat{\mathbb P}_{\boldsymbol\theta}
\right)
-
\gamma C_{\mathcal Z}
\right]
\right\}.
\]
Consequently, since the infimum is no larger than the value at any
fixed $\gamma$, the following upper bound also holds for every
$\gamma\ge0$:
\[
            \sup_{\mathbb{Q} \in \mathcal{B}_\rho(\widehat{\mathbb{P}}_N)} S_{\epsilon, \tau}\left(\mathbb{Q}, \, \widehat{\mathbb{P}}_{\boldsymbol{\theta}}\right)
            \leq
\gamma\rho+
\sup_{\mathcal Z\in\mathbb R^{N\times D\times T}}
\left[
S_{\epsilon,\tau}
\left(
\widehat{\mathbb Q}_{\mathcal Z},
\widehat{\mathbb P}_{\boldsymbol\theta}
\right)
-
\gamma C_{\mathcal Z}
\right],
\]

This completes the proof.
\end{proof}

\begin{proof}[Proof of Corollary~\ref{cor:fixed-gamma-drio}]

By \cref{thm:dual-formulation}, for any fixed imputer parameter
$\boldsymbol{\theta}$, the worst-case Sinkhorn term satisfies
\[
\sup_{\mathbb{Q}\in\mathcal{B}_\rho(\widehat{\mathbb{P}}_N)}
S_{\epsilon,\tau}
\left(
\mathbb{Q},
\widehat{\mathbb{P}}_{\boldsymbol{\theta}}
\right)
\le
\inf_{\gamma\ge 0}
\left\{
\gamma\rho
+
\sup_{\mathcal Z}
\left[
S_{\epsilon,\tau}
\left(
\widehat{\mathbb Q}_{\mathcal Z},
\widehat{\mathbb P}_{\boldsymbol\theta}
\right)
-
\gamma C_{\mathcal Z}
\right]
\right\}.
\]
Since the infimum over $\gamma\ge0$ is upper bounded by evaluating the
objective at any $\gamma\ge0$, we have that for any $\gamma\ge0$:
\[
\sup_{\mathbb{Q}\in\mathcal{B}_\rho(\widehat{\mathbb{P}}_N)}
S_{\epsilon,\tau}
\left(
\mathbb{Q},
\widehat{\mathbb{P}}_{\boldsymbol{\theta}}
\right)
\le
\gamma\rho
+
\sup_{\mathcal Z}
\left\{
S_{\epsilon,\tau}
\left(
\widehat{\mathbb Q}_{\mathcal Z},
\widehat{\mathbb P}_{\boldsymbol\theta}
\right)
-
\gamma C_{\mathcal Z}
\right\}.
\]
Multiplying both sides by $(1-\alpha)\ge0$ and adding
$\alpha R_{\boldsymbol{\theta}}$ yields 
\begin{equation}
\alpha R_{\boldsymbol{\theta}}
+
(1-\alpha)
\sup_{\mathbb{Q}\in\mathcal{B}_\rho(\widehat{\mathbb{P}}_N)}
S_{\epsilon,\tau}
\left(
\mathbb{Q},
\widehat{\mathbb{P}}_{\boldsymbol{\theta}}
\right)
\nonumber
\le
\alpha R_{\boldsymbol{\theta}}
+
(1-\alpha)
\left[
\gamma\rho
+
\sup_{\mathcal Z}
\left\{
S_{\epsilon,\tau}
\left(
\widehat{\mathbb Q}_{\mathcal Z},
\widehat{\mathbb P}_{\boldsymbol\theta}
\right)
-
\gamma C_{\mathcal Z}
\right\}
\right].
\end{equation}
Finally,
for fixed $\gamma$ and $\rho$, the term $\gamma\rho$ is constant with
respect to both $\boldsymbol{\theta}$ and $\mathcal Z$. Therefore,
minimizing the fixed-penalty upper bound over $\boldsymbol{\theta}$ is
equivalent to dropping this constant term, which gives
\cref{eq:final-loss}.

This completes the proof.

\end{proof}

\section{Experimental Details}\label{sec:experiment-details}
\subsection{Datasets}

We evaluate on seven multivariate time series datasets spanning diverse domains, including healthcare, transportation, environmental monitoring, industrial systems, and human activity. The CMAPSS dataset is obtained from NASA, PhysioNet from the PhysioNet Challenge 2012, and the remaining 5 datasets from the UCI Machine Learning Repository. Below we describe each dataset and our preprocessing procedures. In model training, validation, and testing pipeline, we follow \citet{tashiro2021csdi} and structure all datasets as three-dimensional tensors of shape $(N, T, D)$ representing samples, time steps, and features, respectively. Note that the exchange of temporal and feature dimensions does not affect our theory and algorithm as one just needs to swap the indices during computation.

\paragraph{CNNpred \citep{hoseinzade2019cnnpred}.} UCI stock market data combining 5 US indices (S\&P 500, NASDAQ, DJI, Russell, NYSE) from 2010--2017. Features include technical indicators, commodity prices, treasury rates, and global market indices. We segment each index into non-overlapping quarterly chunks of 60 trading days, treating each chunk as an independent sample. The resulting shape is $(165, 60, 78)$ representing 165 quarters pooled across the 5 indices.

\paragraph{PEMS08 \citep{guo2019attention}.} California highway traffic sensor data from Districts 8. Raw measurements (flow, occupancy, speed) are collected at 5-minute intervals. We aggregate to daily resolution by summing flow counts and averaging occupancy and speed, treating each sensor as an independent sample with days as the temporal dimension. The resulting shapes are $(170, 62, 3)$ for PEMS08, representing 170 sensors over 62 days, respectively.

\paragraph{PM2.5 \citep{liang2015assessing}.} Beijing air quality data (2010--2014) with hourly measurements of PM2.5 concentration and meteorological variables (dew point, temperature, pressure, wind speed, cumulative snow, cumulative rain). We structure the data with weeks as the sample unit, where each week contains 168 hourly observations. The resulting shape is $(260, 168, 7)$ representing 260 weeks.

\paragraph{Gas Sensor \citep{ziyatdinov2015bioinspired}.} UCI Gas Sensor Array data from 58 experiments exposing 16 metal-oxide sensors to acetone/ethanol mixtures under flow modulation. Each experiment contains 7500 time steps recorded at 25Hz. We downsample by a factor of 10 to reduce temporal redundancy, then segment each experiment into 5 non-overlapping chunks of 150 time steps. Each chunk is treated as an independent sample. The resulting shape is $(290, 150, 16)$ representing $58 \times 5 = 290$ chunks.

\paragraph{CMAPSS (url: \url{https://data.nasa.gov/dataset/cmapss-jet-engine-simulated-data}).} NASA turbofan engine degradation simulation combining the FD001--FD004 subsets. Each engine's operational lifetime constitutes one sample, with operational cycles as the temporal dimension. To ensure uniform sequence length, we retain only engines with at least 207 cycles (the median across all subsets) and truncate to the first 207 cycles. The 21 features comprise sensor measurements tracking engine health degradation. The resulting shape is $(359, 207, 21)$ representing 359 engines.

\paragraph{HAR \citep{anguita2013public}.} UCI Human Activity Recognition dataset containing smartphone inertial sensor readings from 30 subjects performing 6 activities. Each sample corresponds to a 2.56-second sliding window sampled at 50Hz, yielding 128 time steps. The 9 features comprise triaxial body acceleration, body gyroscope, and total acceleration. We use the test partition, resulting in shape $(2947, 128, 9)$.

\paragraph{PhysioNet (url: \url{https://physionet.org/content/challenge-2012/1.0.0/}).} ICU patient monitoring data from the PhysioNet Challenge 2012. Each patient's 48-hour ICU stay is treated as one sample, with 35 clinical variables (vital signs, lab values) aggregated to hourly resolution. This dataset exhibits substantial natural missingness due to irregular clinical measurement schedules. The resulting shape is $(4000, 48, 35)$ where 4000 is the number of patients.

\subsection{Missing Data Generation Mechanisms}\label{sec:missingness}
We implement two artificial missingness mechanisms to evaluate imputation methods.

\paragraph{Missing Completely at Random (MCAR).}
Under MCAR, the probability of a value being missing is independent of both the observed and unobserved data. For each sample $i$, we generate the ground-truth mask $\boldsymbol{M}^{(i)}_{\mathrm{gt}}$ according to the raw observation mask $\boldsymbol{M}^{(i)}$ and the target missing ratio $r$. Specifically, we first identify all observed indices $\mathcal{I}^{(i)} = \{(t,d) : \boldsymbol{M}^{(i)}_{t,d} = 1\}$, then uniformly sample $\lfloor r\; |\mathcal{I}^{(i)}| \rfloor$ indices to mask. The ground-truth mask thus is set to $\boldsymbol{M}^{(i)}_{\mathrm{gt}, d,t} = 0$ for sampled indices and $\boldsymbol{M}^{(i)}_{\mathrm{gt}, d,t} = 1$  otherwise.

\paragraph{Missing Not at Random (MNAR).}
Under MNAR, the probability of missingness depends on the unobserved value itself. We implement a mechanism where extreme values (in either direction) are more likely to be missing. Specifically, for each entry with value $\boldsymbol{X}^{(i)}_{d,t}$, we compute its z-score $z_{d,t}^{(i)} = (\boldsymbol{X}^{(i)}_{d,t} - \overline{\boldsymbol{x}}_d) / \sigma_d$ where $\overline{\boldsymbol{x}}_d$ and $\sigma_d$ are feature mean and standard deviation across sample-temporal space. The missing weigh for each observed value is then given by $w_{t,d}^{(i)} = \Phi(|z_{t,d}^{(i)}|)$, where $\Phi(\cdot)$ is the CDF of the standard normal distribution. Then, we normalize these weights across samples to obtain the missing probabilities $p_{t,d}^{(i)} = w^{(i)}_{t,d} /\sum_i \sum_{(t',d') \in \mathcal{I}^{(i)}} p_{t',d'}^{(i)}$. Consequently extreme values have a higher missing probability, a realistic scenario in many domains such as sensor saturation or reporting bias. Finally, we sample $\lfloor r \cdot |\mathcal{I}^{(i)}| \rfloor$ indices without replacement according to missing ratio $r$ and set the ground truth mask $\boldsymbol{M}^{(i)}_{\mathrm{gt}, d,t} = 0$ for sampled indices and $\boldsymbol{M}^{(i)}_{\mathrm{gt}, d,t} = 1$  otherwise.

\paragraph{Missingness Settings.}
We evaluate at three missing ratios $r \in \{10, 50, 90\}\%$, representing easy, moderate, and challenging imputation scenarios. For datasets with natural missingness (e.g., PhysioNet), the synthetic missing mask is applied on top of the original observed values, and evaluation is performed only on the synthetically masked entries where ground truth is available. To ensure reproducibility and fair comparison, we use deterministic seed generation to guarantee identical test sets across different experimental configurations.

\subsection{Hyperparameter Configuration for \objname}
\label{sec:hyperparams}
Our proposed training objective \eqref{eq:final-loss} incorporates four hyperparameters: the unbalanced Sinkhorn divergence parameters ($\varepsilon$ and $\tau$) and the distributional robustness parameters ($\alpha$ and $\gamma$). While performing cross-validation across the full four-dimensional parameter space could potentially enhance model performance, it incurs prohibitive computational overhead. We therefore fix the divergence parameters and employ cross-validation exclusively to tune $\alpha$ and $\gamma$.

\paragraph{Unbalanced Sinkhorn Divergence Parameters ($\epsilon$, $\tau$).}
The divergence parameters determine the geometry and smoothness of the transport objective. The entropic regularization $\epsilon$ smooths the transport plan and enables stable gradients, while the marginal-relaxation parameter $\tau$ controls the penalty for local mass creation and destruction. Smaller $\tau$ allows more flexible mass mismatch, whereas larger $\tau$ approaches balanced Sinkhorn with stricter mass preservation. Since $\tau$ affects the degree of marginal relaxation and $\gamma$ controls the transport-cost penalty in the adversarial search, tuning both extensively can introduce redundant complexity. We therefore fix $\epsilon$ by the heuristic described below and use $\tau=10$ as the default finite-relaxation setting for all data set and missing scenarios. This choice softly encourages mass conservation while still allowing local mass fluctuations caused by non-stationarity and systematic missingness. In Appendix~\ref{sec:tau_sensitivity}, we provide a sensitivity analysis showing that finite $\tau$ values are stable and consistently outperform the balanced Sinkhorn variant, supporting $\tau=10$ as a robust default.

As for $\varepsilon$, we adopt the adaptive scheme from \cite{muzellec2020missing}. That is, given a data batch, we compute pairwise squared Euclidean distances between samples, then set $\varepsilon = 0.05 \cdot q_{0.5}$, where $q_{0.5}$ is the median of the non-zero distances. This ensures $\varepsilon$ scales appropriately with data geometry across datasets without manual tuning.
This strategy effectively fixes the divergence landscape, allowing us to efficiently optimize the trade-off between fidelity and robustness via $\alpha$ and $\gamma$. In the following, we detail the cross-validation procedure for these parameters.

\paragraph{Cross-Validation for Robustness Parameters ($\alpha$, $\gamma$).}
For each hyperparameter combination, we train the imputer using \eqref{eq:final-loss} on the training set and evaluate performance on the validation set. Following standard practice and compute the mean squared error on the \textit{observed entries} of the validation data as the selection criterion, choosing the $(\alpha, \gamma)$ pair that achieves the lowest validation loss.

Because the reconstruction loss and worst-case divergence can have different numerical scales, we normalize each term in \eqref{eq:final-loss} by an online running average of its own batch value during training. This brings both terms to comparable $\mathcal{O}(1)$ scale and makes the reconstruction--robustness trade-off induced by $\alpha$ more interpretable across datasets. We note that with a sufficiently dense $\alpha$ grid, the raw and normalized objectives can represent similar trade-offs as the normalization only re-weight the two objectives. However, under the realistic scenario with a coarse grid in cross-validation, normalization stabilizes the grid search and reduces sensitivity to dataset-specific loss magnitudes. 

We sweep $\alpha \in \{0.25, 0.5, 0.75, 0.9\}$ and $\gamma \in \{0.1, 1.0, 5.0, 10.0\}$, totaling $16$ configurations per missingness scenario. After selecting $(\alpha^*,\gamma^*)$, we retrain the imputer on the union of the training and validation sets before evaluating on the test set.

This validation strategy does not require access to naturally missing ground truth at deployment. Instead, it tests whether the imputer can reconstruct artificially held-out observed entries in unseen validation samples, thereby measuring its ability to learn spatiotemporal dependencies that generalize from observed to unobserved entries. The selected $(\alpha,\gamma)$ therefore balances reconstruction accuracy and robustness to distributional shift while reducing overfitting to potentially biased empirical observations.

\subsection{Evaluation Metrics}
\label{sec:error-metric}

We evaluate imputation quality using both point-wise and distributional metrics. 
Let $\mathcal{S}$ denote the evaluation sample set. For each sample
$i\in\mathcal{S}$, let $\boldsymbol{X}^{(i)}\in\mathbb{R}^{D\times T}$
be the ground-truth trajectory and $\widehat{\boldsymbol{X}}^{(i)}$ be
the imputed trajectory. Recall that $\boldsymbol{M}^{(i)}$ denotes the
raw observation mask, where $M^{(i)}_{d,t}=1$ if ground truth is
available in the original data, and $\boldsymbol{M}_{\mathrm{gt}}^{(i)}$
denotes the model-observed mask after artificial masking. Thus, the
artificial evaluation mask is
$
\boldsymbol{M}_{\mathrm{eval}}^{(i)}
\coloneqq
\boldsymbol{M}^{(i)}-\boldsymbol{M}_{\mathrm{gt}}^{(i)}
=
\boldsymbol{M}^{(i)}\odot(1-\boldsymbol{M}_{\mathrm{gt}}^{(i)}),
$
where $\overline{M}_{\mathrm{gt},d,t}^{(i)}=1$ indicates an entry that
was originally observed but artificially held out for evaluation. Entries
observed by the model and naturally missing entries are not directly
evaluated. The definitions and evaluation on artificially masked entries follow standard metrics calculation procedure for MTS imputation, see, for example, \citep{tashiro2021csdi}.
Let
$
m_i=\|\boldsymbol{M}_{\mathrm{eval}}^{(i)}\|_0 >0
$
be the number of artificially held-out entries for sample $i$.

\paragraph{Mean Squared Error (MSE).}
MSE measures point-wise reconstruction accuracy on artificially held-out entries:
\begin{equation}
\tag{MSE}
\label{eq:mse}    
    \mathrm{MSE}
    =
    \frac{
    \sum_{i\in\mathcal{S}}
    \left\|
    \boldsymbol{M}_{\mathrm{eval}}^{(i)}
    \odot
    \left(
    \boldsymbol{X}^{(i)}-\widehat{\boldsymbol{X}}^{(i)}
    \right)
    \right\|_F^2
    }{
    \sum_{i\in\mathcal{S}} m_i
    }.
\end{equation}

Thus, MSE evaluates whether the imputer accurately recovers entries that were hidden during evaluation. It does not measure distributional alignment of the imputed entries and the masked ground truth entries.

\paragraph{Squared Maximum Mean Discrepancy (MMD$^2$).}
To evaluate joint feature-temporal distributional alignment, we compute squared maximum mean discrepancy (MMD$^2$) between masked trajectory vectors. For each sample with $m_i>0$, define the masked vector of imputed and ground truth samples as
\[
    \boldsymbol{u}_i
    \coloneqq
    \frac{
    \mathrm{vec}\!\left(
    \boldsymbol{M}_{\mathrm{eval}}^{(i)}
    \odot
    \widehat{\boldsymbol{X}}^{(i)}
    \right)
    }{\sqrt{m_i}},
    \qquad
    \boldsymbol{v}_i
    \coloneqq
    \frac{
    \mathrm{vec}\!\left(
    \boldsymbol{M}_{\mathrm{eval}}^{(i)}
    \odot
    \boldsymbol{X}^{(i)}
    \right)
    }{\sqrt{m_i}},
\]
respectively, where entries outside the evaluation mask are zero-padded in both vectors. We note that these zeros are not observed values being evaluated and they only embed each sample into a common feature--temporal vector space. The normalization by $\sqrt{m_i}$ makes pairwise distances more comparable across samples with different numbers of held-out entries. We then compute the MMD$^2$ using standard procedure in the generative modeling community \citep{gretton2012kernel, li2015generative, sutherland2016generative}. That is,
\begin{equation}
    \tag{MMD$^2$}
    \label{eq:MMD2}
    \mathrm{MMD}^2
    =
    \frac{1}{|\mathcal{S}|^2}
    \sum_{i,j\in\mathcal{S}}
    \kappa(\boldsymbol{u}_i,\boldsymbol{u}_j)
    +
    \frac{1}{|\mathcal{S}|^2}
    \sum_{i,j\in\mathcal{S}}
    \kappa(\boldsymbol{v}_i,\boldsymbol{v}_j)
    -
    \frac{2}{|\mathcal{S}|^2}
    \sum_{i,j\in\mathcal{S}}
    \kappa(\boldsymbol{u}_i,\boldsymbol{v}_j),
\end{equation}
where $\kappa$ is a multi-scale RBF kernel defined by
\[
    \kappa(\boldsymbol{x},\boldsymbol{y})
    =
    \frac{1}{|\mathcal{A}|}
    \sum_{a\in\mathcal{A}}
    \exp\left(
    -
    \frac{
    \|\boldsymbol{x}-\boldsymbol{y}\|_2^2
    }{
    2a\sigma_{\mathrm{med}}^2
    }
    \right),
\]
where $
    \mathcal{A}=\{0.25,0.5,1,2,4\}.
$ 
The bandwidth $\sigma_{\mathrm{med}}^2$ is selected by the median heuristic over the pooled vectors
$
\{\boldsymbol{u}_i\}_{i\in\mathcal{S}}
\cup
\{\boldsymbol{v}_i\}_{i\in\mathcal{S}}
$ \citep{li2015generative, sutherland2016generative}.
This metric compares the empirical distribution of imputed held-out trajectories with that of the true held-out trajectories, thereby preserving more joint feature--temporal structure than metrics that flatten all entries into one scalar distribution.

\paragraph{One-Dimensional Wasserstein-2 Distance.}
Additionally, we also report a one-dimensional Wasserstein-2 distance (1-D W2) as an auxiliary distributional metric. This metric aggregates all artificially held-out imputed values into one empirical distribution and all corresponding ground-truth values into another. Specifically, let
$
\{\widehat{y}_{\ell}\}_{\ell=1}^{K}
$
be the flattened collection of $\widehat{X}^{(i)}_{d,t}$ over all entries with $M^{(i)}_{d,t}=1$, and let
$
\{y_{\ell}\}_{\ell=1}^{K}
$
be the corresponding flattened collection of $X^{(i)}_{d,t}$, where
$
K=\sum_{i\in\mathcal{S}}m_i.
$
Let
$
\widehat{y}_{(1)}\le\cdots\le\widehat{y}_{(K)}
$
and
$
y_{(1)}\le\cdots\le y_{(K)}
$
denote the sorted values. We compute
\begin{equation}
    \tag{$1$-D W2}
    \label{eq:1dw2}
    W^{\mathrm{1D}}_2
    =
    \left[
    \frac{1}{K}
    \sum_{\ell=1}^{K}
    \left(
    \widehat{y}_{(\ell)}-y_{(\ell)}
    \right)^2
    \right]^{1/2}.
\end{equation}
The sorting step is essential because the closed-form one-dimensional Wasserstein distance compares empirical quantiles. While this metric captures marginal distributional alignment of held-out values, it does not preserve sample, feature, and temporal dependence. Therefore, we use it as a complementary metric to evaluate global distributional alignment on held-out entries.

\paragraph{Wasserstein-Fourier Distance.}
We further evaluate whether imputations preserve temporal frequency structure using the Wasserstein-Fourier (WF) distance~\citep{cazelles2020wasserstein}. WF compares time series by first representing each series through its normalized power spectral density and then computing a Wasserstein distance between the resulting spectral distributions. This metric captures how spectral energy is displaced across frequencies, and is therefore useful for evaluating periodicity, seasonality, and other frequency-domain structures.

For each sample--feature pair $(i,d)$, we construct the completed trajectory
$
\widehat{\boldsymbol{x}}_{i,d}
$
by keeping entries observed by the model and filling missing entries with the imputed values. We compare it with the ground-truth trajectory
$
\boldsymbol{x}_{i,d}
$
only when all time steps for that sample--feature pair have ground truth in the raw data. For nonnegative frequency grid
$
\omega_k = k/T,\; k=0,\ldots,\lfloor T/2\rfloor
$, we compute one-sided periodograms using the real FFT:
\[
    P_{i,d,k}
    =
    \left|
    \sum_{t=1}^{T}
    x_{i,d,t}
    \exp(-2\pi\sqrt{-1}\,kt/T)
    \right|^2,
    \qquad
    \widehat{P}_{i,d,k}
    =
    \left|
    \sum_{t=1}^{T}
    \widehat{x}_{i,d,t}
    \exp(-2\pi\sqrt{-1}\,kt/T)
    \right|^2.
\]
The constant factor in the periodogram is omitted because each spectrum is normalized. Then, we normalize each periodogram into a probability distribution over frequencies:
\[
    p_{i,d,k}
    =
    \frac{P_{i,d,k}}{\sum_{\ell}P_{i,d,\ell}},
    \qquad
    \widehat{p}_{i,d,k}
    =
    \frac{\widehat{P}_{i,d,k}}{\sum_{\ell}\widehat{P}_{i,d,\ell}}.
\]
If either spectrum has zero total power, the corresponding pair is excluded. 
The WF distance for $(i,d)$ is then computed as the one-dimensional Wasserstein-2 distance between the two normalized spectra on the frequency grid:
\begin{equation}
    \tag{pair-wise WF}
    \label{eq:pairWF}
    \mathrm{WF}_{i,d}
    =
    W_2
    \left(
    \sum_k \widehat{p}_{i,d,k}\delta_{\omega_k},
    \sum_k p_{i,d,k}\delta_{\omega_k}
    \right).
\end{equation}
Here, $W_2$ is computed by monotone coupling on the frequency axis.  
Note that \eqref{eq:pairWF} compares two weighted discrete distributions over fixed frequency bins, where the weights are the normalized spectral powers, while  \eqref{eq:1dw2} compares two equally weighted empirical distributions of scalar held-out values via sorted order statistics.
Finally, for each sample, we average $\mathrm{WF}_{i,d}$ over all valid features, and then report the mean across valid samples: 
\begin{equation}
    \tag{WF}
    \label{eq:WF}
    \mathrm{WF}
    =
    \frac{1}{|\mathcal{S}_{\mathrm{freq}}|}
    \sum_{(i,d)\in\mathcal{S}_{\mathrm{freq}}}
    \mathrm{WF}_{i,d},
\end{equation}
where $\mathcal{S}_{\mathrm{freq}}$ is the set of valid sample-feature pairs (sample-feature sequences with full ground-truth trajectory available for all $T$ time steps).

\subsection{Configuration for Ablation Study}
\label{app:model_backbones}
\subsubsection{Model Backbone Architecture}
 In principle, \eqref{eq:final-loss} can be paired with any differentiable imputer backbone that maps partially observed trajectories and masks to completed trajectories. In practice, as shown in Table~\ref{tab:ablation}, performance depends on the backbone architecture; our main experiments therefore use SAITS, which is an imputation-specific attention backbone. The imputer $G_{\boldsymbol{\theta}}$ takes two inputs: the partially observed data $\boldsymbol{X}_{\mathrm{obs}}^{(i)} \in \mathbb{R}^{D \times T}$, with missing entries initialized to the per-position mean across the current batch $\overline{\boldsymbol{X}}^{(i)}$, and the binary observation mask $\boldsymbol{M}^{(i)} \in \{0,1\}^{D \times T}$. The data and mask are paired and passed to the backbone (the exact layout is backbone-specific, e.g. $(B, T, 2D)$ for the MLP/LSTM and $(B, D, T, 2)$ for the STT). For each sample $i$, the imputer outputs a deterministic reconstruction $\widehat{\boldsymbol{X}}^{(i)} \in \mathbb{R}^{D \times T}$, so the batched output tensor has shape $\mathbb{R}^{B \times D \times T}$. 

In \cref{sec:numerical} (\cref{tab:ablation}) we ablate 4 standard backbones of varying model design feature and complexities, including  MLP, LSTM, and STT, and SAITS \citep{du2023saits}. We summarize these models in the following.


\paragraph{Multi-Layer Perceptron (MLP).} A simple feedforward network that processes each time step independently. The input, i.e., concatenated data and mask at time $t$, is passed through 3 fully-connected layers of hidden dimension 128 with Rectified Linear Unit (ReLU) activations and dropout, followed by a linear output layer producing an $\R^D$ vector. Missing entries are initialized via per-position mean imputation. This baseline captures no temporal dependencies as it processes each timestamp separately.

\paragraph{Long Short-Term Memory (LSTM).} We implement a bidirectional LSTM that captures temporal dependencies. The input sequence $[\boldsymbol{X}_{\text{obs}}; \boldsymbol{M}] \in \R^{T \times 2D}$ is first projected to dimension 128, then processed by a 2-layer bidirectional LSTM with hidden size 128. The concatenated forward and backward hidden states (dimension 256) are projected back to $D$ features. This architecture captures long-range temporal patterns but lacks an explicit cross-feature attention mechanism, relying solely on the LSTM gates to mix feature information.

\paragraph{Spatiotemporal Transformer (STT).} Our primary architecture employs factorized attention over both temporal and spatial (feature) dimensions. The input $[\boldsymbol{X}_{\text{obs}}; \boldsymbol{M}] \in \R^{D \times T \times 2}$ is projected to dimension 128, and learnable positional encodings are added along both the temporal and spatial axes. Each of the 4 transformer layers applies both temporal and spatial attention: temporal attention performs self-attention across $T$ time steps independently for each feature, with complexity $\mathcal{O}(DT^2)$; spatial attention performs self-attention across $D$ features independently for each time step, with complexity $\mathcal{O}(TD^2)$. A final LayerNorm and linear projection map each $(d, t)$ latent vector to a scalar prediction.

Each attention block uses 8 heads, pre-layer normalization, Gaussian Error Linear Unit (GELU) activation in the feed-forward network (dimension 512), and residual connections. The factorized design reduces complexity from $\mathcal{O}((DT)^2)$ for full attention to $\mathcal{O}(DT^2 + TD^2)$, enabling scalability to larger spatiotemporal grids while capturing both intra-feature temporal dynamics and inter-feature correlations at each time step.

\paragraph{Self-Attention-based Imputation for Time Series (SAITS).}
SAITS~\citep{du2023saits} is an imputation-specific attention backbone designed to reconstruct missing entries from observed temporal context. Overall, SAITS takes the concatenated data and mask $[\boldsymbol{X}_{\mathrm{obs}};\boldsymbol{M}] \in \mathbb{R}^{B\times T\times 2D}$ as input and outputs a completed trajectory $\widehat{\boldsymbol{X}}\in\mathbb{R}^{B\times D\times T}$. 

The input is projected to $d_{\mathrm{model}}=256$ and augmented with sinusoidal temporal positional encodings. The first stack of diagonally-masked self-attention (DMSA) blocks produces an initial reconstruction $\widetilde{\boldsymbol{X}}_1$. Then, observed entries are restored to form
$
\boldsymbol{X}'=\boldsymbol{M}\odot\boldsymbol{X}_{\mathrm{obs}}+(1-\boldsymbol{M})\odot\widetilde{\boldsymbol{X}}_1.
$
A second DMSA stack refines this filled trajectory into $\widetilde{\boldsymbol{X}}_2$ and produces attention information used by a learned combination block, which adaptively fuses $\widetilde{\boldsymbol{X}}_1$ and $\widetilde{\boldsymbol{X}}_2$ into the final reconstruction. The diagonal attention mask can prevent each position from simply copying itself, forcing reconstruction from surrounding temporal context. 

In our main experiments, SAITS provides the backbone architecture for $\objname$, while its original MAE objective with internal masking is included separately in the objective ablation. This separation lets us distinguish the benefit of the proposed distributionally robust objective from the architectural strength of SAITS itself.

\subsubsection{Training Objectives} For objective functions, we compare $\objname$ with B-DRIO ($\objname$ with balanced Sinkhorn formulation), MSE, and the original SAITS procedure. B-DRIO uses balanced optimal transport with strict mass conservation which is equivalent to its  unbalanced variant ($S_{\epsilon, \tau}$) with $\tau \to \infty$, which enforces $\pi_1 = \mu$ and $\pi_2 = \nu$ exactly (See Appendix~\ref{sec:discuss-loss} for details). The original SAITS procedure combines 
point-wise MAE reconstruction with an additional loss on internal held-out set during training to prevent over-fitting \citep{du2023saits}, which could be viewed as an alternative way to induce robustness.

\subsubsection{Other Training Details}
For computational efficiency, we conduct this ablation study on 5 representative datasets, including CNNpred, PEMS08, PM2.5, Gas Senaor, and CMAPSS, under MCAR and MNAR with all missing ratios (10\%, 50\%, and 90\%), ensuring the comparison generalizes across varying data characteristics and missingness scenarios.

All models are trained with the Adam optimizer using learning rate $5 \times 10^{-4}$ and weight decay $10^{-6}$ with batch size 32. For the models with $\objname$-type objective, we use inner learning rate 0.01 and $K$ = 8 for the adversarial weight update across all dataset and missing scenarios. The number of inner ascent steps $K$ controls the strength of the adversarial approximation \citep{sinha2017certifying}. Small $K$ gives a weak but stable adversary and reduces computation, while larger $K$ more closely approximates the inner supremum at higher cost. In practice, overly large $K$ may produce adversarial trajectories that make the outer update noisy or overly conservative. We note that the effects of $K$ are mitigated by the cross-validation  procedure since it tunes the robustness based on the fixed parameters, e.g., $K$ and $\tau$. Therefore, we fix $K=8$ throughout the paper for all $\objname$-type objectives across all dataset and missingness scenarios. 

\subsection{Configuration for Downstream Forecasting Tasks}\label{sec:downstream_task}
Since MTS imputation aims to serve the downstream tasks, e.g., forecasting with the imputed data, we evaluate whether improved imputation quality translates into better predictive usage. For each of the seven datasets and each $(\text{mcar}/\text{mnar}) \times \{10\%,50\%,90\%\}$ setting, we first split each completed sample along the time dimension into train/validation/test segments with a $70\%/10\%/20\%$ partition. Then, we train a 2-layer LSTM model on the training segment using next-step prediction, with hidden dimension $128$, dropout $0.1$, Adam optimizer with learning rate $10^{-3}$, and batch size $64$. The forecaster is rolled out autoregressively on validation horizon for model selection, then the selected model is used to forecast the test horizon autoregressively. The forecasts are evaluated by MSE against the masked ground-truth future trajectory. Since the same forecasting architecture and protocol are used for all imputers, this experiment isolates how different imputed series affect downstream forecasting performance.

\subsection{Additional Numerical Results and Details}
\label{sec:extended-results}

We provide additional numerical results and implementation details for the proposed distributionally robust imputer objective in \eqref{eq:final-loss}. In \cref{sec:more_imputation_performance}, we report full imputation results under both MCAR and MNAR settings to complement \cref{tab:wide_mnar} in \cref{sec:numerical}, including evaluations under MSE \eqref{eq:mse}, MMD$^2$ \eqref{eq:MMD2}, one-dimensional Wasserstein-2 distance \eqref{eq:1dw2}, and Wasserstein-Fourier distance \eqref{eq:WF}. We further analyze the deployable cross-validation strategy in \cref{sec:cv_compare}, the sensitivity to the unbalanced Sinkhorn mass-relaxation parameter in \cref{sec:tau_sensitivity}, and the computational trade-off in \cref{sec:run_time}.

All methods use the same train/validation/test splits, with hyperparameters selected according to each method's default validation criterion when available. For external baselines, we use the authors' recommended configurations from official code repositories whenever possible. Due to computational cost, we train diffusion-based methods (SSSD, CSDI) for at most 30 epochs, deep learning methods for at most 65 epochs, and iterative methods (MF, MDOT) for at most 1000 iterations. All experiments are run on a single RTX-6000 GPU.

\subsubsection{More on Imputation Performance}\label{sec:more_imputation_performance}

\begin{table*}[thb]
\centering
\caption{Per-dataset \ref{eq:mse} and \ref{eq:MMD2} under MCAR missingness. Each cell shows the mean with the standard deviation in parentheses, across all missing ratios. Data normalization procedure, highlighting, and $\objname$ backbone all follow same setting described in \cref{tab:ablation}.}
\label{tab:appendix_mcar_mse_mmd_stacked}
\begin{threeparttable}

\begin{adjustbox}{max width=0.85\textwidth}
\begin{tabular}{l c c c c c c c}
\toprule
Method & CNN & PeMS & PM2.5 & Gas & CMAP & HAR & Physio \\
\midrule
\multicolumn{8}{l}{\textbf{MSE (point-wise reconstruction, lower better) }} \\
\midrule
\multicolumn{8}{l}{\textit{Baselines}} \\
Mean & $1.110_{(0.144)}$ & $1.274_{(0.458)}$ & $1.530_{(0.329)}$ & $1.087_{(0.232)}$ & $1.079_{(0.079)}$ & $1.013_{(0.052)}$ & $1.026_{(0.036)}$ \\
MF & $1.453_{(0.274)}$ & $2.425_{(3.323)}$ & $1.999_{(0.374)}$ & $0.442_{(0.160)}$ & $1.083_{(0.157)}$ & $1.680_{(0.177)}$ & $1.961_{(0.415)}$ \\
\midrule
\multicolumn{8}{l}{\textit{Benchmarks}} \\
CSDI & $0.759_{(0.020)}$ & $13.731_{(11.664)}$ & $1.098_{(0.360)}$ & $0.658_{(0.413)}$ & $\mathit{0.079}_{(0.131)}$ & $\mathit{0.212}_{(0.227)}$ & $\mathbf{0.473}_{(0.146)}$ \\
SSSD & $12.934_{(0.185)}$ & $0.530_{(0.177)}$ & $\underline{0.920}_{(0.126)}$ & $2.223_{(0.777)}$ & $2.076_{(0.209)}$ & $0.280_{(0.329)}$ & $3.403_{(0.634)}$ \\
BRITS & $\mathit{0.533}_{(0.171)}$ & $0.416_{(0.271)}$ & $\mathit{0.933}_{(0.213)}$ & $0.058_{(0.039)}$ & $0.289_{(0.299)}$ & $0.343_{(0.358)}$ & $0.579_{(0.157)}$ \\
SAITS & $\underline{0.466}_{(0.158)}$ & $\underline{0.187}_{(0.118)}$ & $0.978_{(0.328)}$ & $0.022_{(0.019)}$ & $\underline{0.064}_{(0.102)}$ & $\underline{0.201}_{(0.202)}$ & $\mathit{0.520}_{(0.144)}$ \\
IF & $0.626_{(0.107)}$ & $\mathit{0.221}_{(0.083)}$ & $1.075_{(0.385)}$ & $\mathbf{0.011}_{(0.004)}$ & $0.152_{(0.180)}$ & $0.369_{(0.319)}$ & $\underline{0.502}_{(0.202)}$ \\
nMW & $2.578_{(3.173)}$ & $3.311_{(3.806)}$ & $3.097_{(2.427)}$ & $2.067_{(3.260)}$ & $0.627_{(0.750)}$ & $2.418_{(2.480)}$ & $7.855_{(9.875)}$ \\
MDOT & $0.887_{(0.335)}$ & $0.636_{(0.502)}$ & $1.175_{(0.176)}$ & $0.316_{(0.283)}$ & $0.688_{(0.226)}$ & $0.636_{(0.190)}$ & $0.851_{(0.181)}$ \\
PSW & $1.058_{(0.200)}$ & $0.278_{(0.152)}$ & $\mathbf{0.222}_{(0.203)}$ & $\underline{0.011}_{(0.017)}$ & $0.871_{(0.125)}$ & $0.340_{(0.283)}$ & $0.622_{(0.122)}$ \\
\midrule
\multicolumn{8}{l}{\textit{Ours}} \\
DRIO & $\mathbf{0.446}_{(0.157)}$ & $\mathbf{0.170}_{(0.105)}$ & $0.938_{(0.317)}$ & $\mathit{0.018}_{(0.012)}$ & $\mathbf{0.058}_{(0.091)}$ & $\mathbf{0.197}_{(0.210)}$ & $0.522_{(0.148)}$ \\
\midrule
\multicolumn{8}{l}{\textbf{MMD$^2$ (joint feature-temporal distributional alignment, lower better) }} \\
\midrule
\multicolumn{8}{l}{\textit{Baselines}} \\
Mean & $0.201_{(0.251)}$ & $0.253_{(0.295)}$ & $0.218_{(0.259)}$ & $0.296_{(0.321)}$ & $0.194_{(0.215)}$ & $0.153_{(0.167)}$ & $0.151_{(0.160)}$ \\
MF & $0.180_{(0.193)}$ & $0.123_{(0.213)}$ & $0.139_{(0.150)}$ & $0.103_{(0.158)}$ & $0.160_{(0.167)}$ & $0.107_{(0.107)}$ & $0.062_{(0.069)}$ \\
\midrule
\multicolumn{8}{l}{\textit{Benchmarks}} \\
CSDI & $0.100_{(0.174)}$ & $0.084_{(0.083)}$ & $0.085_{(0.146)}$ & $0.018_{(0.031)}$ & $\mathbf{0.000}_{(0.000)}$ & $\mathit{0.005}_{(0.008)}$ & $\mathbf{0.012}_{(0.011)}$ \\
SSSD & $0.210_{(0.010)}$ & $0.039_{(0.037)}$ & $0.045_{(0.043)}$ & $0.076_{(0.040)}$ & $0.035_{(0.002)}$ & $0.022_{(0.033)}$ & $0.059_{(0.013)}$ \\
BRITS & $0.089_{(0.155)}$ & $0.109_{(0.188)}$ & $0.084_{(0.146)}$ & $\mathbf{0.000}_{(0.000)}$ & $0.061_{(0.105)}$ & $0.083_{(0.110)}$ & $0.098_{(0.124)}$ \\
SAITS & $\mathit{0.067}_{(0.063)}$ & $\mathit{0.012}_{(0.013)}$ & $\mathit{0.039}_{(0.021)}$ & $0.002_{(0.003)}$ & $\underline{0.002}_{(0.003)}$ & $\underline{0.002}_{(0.003)}$ & $\underline{0.022}_{(0.020)}$ \\
IF & $0.075_{(0.044)}$ & $\mathbf{0.007}_{(0.004)}$ & $0.047_{(0.012)}$ & $\mathit{0.001}_{(0.001)}$ & $0.011_{(0.013)}$ & $0.018_{(0.027)}$ & $0.023_{(0.024)}$ \\
nMW & $0.283_{(0.431)}$ & $0.347_{(0.450)}$ & $0.304_{(0.420)}$ & $0.307_{(0.532)}$ & $0.161_{(0.279)}$ & $0.303_{(0.422)}$ & $0.237_{(0.263)}$ \\
MDOT & $0.161_{(0.253)}$ & $0.157_{(0.271)}$ & $0.149_{(0.258)}$ & $0.148_{(0.256)}$ & $0.095_{(0.164)}$ & $0.109_{(0.106)}$ & $0.137_{(0.160)}$ \\
PSW & $\mathbf{0.013}_{(0.023)}$ & $0.013_{(0.022)}$ & $\mathbf{0.000}_{(0.000)}$ & $\underline{0.000}_{(0.000)}$ & $0.016_{(0.028)}$ & $0.022_{(0.038)}$ & $0.028_{(0.049)}$ \\
\midrule
\multicolumn{8}{l}{\textit{Ours}} \\
DRIO & $\underline{0.056}_{(0.050)}$ & $\underline{0.010}_{(0.011)}$ & $\underline{0.030}_{(0.018)}$ & $0.001_{(0.002)}$ & $\mathit{0.002}_{(0.003)}$ & $\mathbf{0.002}_{(0.003)}$ & $\mathit{0.022}_{(0.021)}$ \\
\bottomrule
\end{tabular}
\end{adjustbox}
\end{threeparttable}
\end{table*}

\begin{table*}[thb]
\centering
\caption{Per-dataset \ref{eq:mse} and \ref{eq:MMD2} under MNAR missingness. Each cell shows the mean with the standard deviation in parentheses, across all missing ratios. Data normalization procedure, highlighting, and $\objname$ backbone all follow same setting described in \cref{tab:ablation}.}
\label{tab:appendix_mnar_mse_mmd_stacked}
\begin{threeparttable}

\begin{adjustbox}{max width=0.85\textwidth}
\begin{tabular}{l c c c c c c c}
\toprule
Method & CNN & PeMS & PM2.5 & Gas & CMAP & HAR & Physio \\
\midrule
\multicolumn{8}{l}{\textbf{MSE (point-wise reconstruction, lower better) }} \\
\midrule
\multicolumn{8}{l}{\textit{Baselines}} \\
Mean & $1.350_{(0.053)}$ & $1.203_{(0.153)}$ & $1.489_{(0.265)}$ & $1.402_{(0.268)}$ & $1.155_{(0.025)}$ & $1.234_{(0.181)}$ & $1.351_{(0.424)}$ \\
MF & $1.670_{(0.058)}$ & $2.368_{(2.759)}$ & $1.763_{(0.353)}$ & $0.565_{(0.307)}$ & $1.164_{(0.142)}$ & $1.895_{(0.064)}$ & $2.258_{(0.349)}$ \\
\midrule
\multicolumn{8}{l}{\textit{Benchmarks}} \\
CSDI & $1.232_{(0.265)}$ & $47.850_{(54.112)}$ & $1.078_{(0.280)}$ & $1.164_{(1.250)}$ & $\mathit{0.136}_{(0.226)}$ & $\mathit{0.287}_{(0.284)}$ & $\mathbf{0.713}_{(0.234)}$ \\
SSSD & $13.309_{(0.412)}$ & $0.596_{(0.136)}$ & $0.996_{(0.181)}$ & $2.446_{(0.690)}$ & $1.904_{(0.638)}$ & $0.379_{(0.442)}$ & $3.819_{(0.847)}$ \\
BRITS & $\mathit{0.764}_{(0.230)}$ & $0.544_{(0.227)}$ & $0.913_{(0.462)}$ & $0.085_{(0.067)}$ & $0.360_{(0.346)}$ & $0.457_{(0.407)}$ & $0.868_{(0.263)}$ \\
SAITS & $\underline{0.680}_{(0.207)}$ & $\mathbf{0.250}_{(0.101)}$ & $\mathit{0.874}_{(0.383)}$ & $0.030_{(0.026)}$ & $\underline{0.104}_{(0.167)}$ & $\underline{0.260}_{(0.253)}$ & $0.775_{(0.272)}$ \\
IF & $0.884_{(0.052)}$ & $\mathit{0.288}_{(0.075)}$ & $1.008_{(0.311)}$ & $\mathbf{0.012}_{(0.001)}$ & $0.160_{(0.240)}$ & $0.424_{(0.284)}$ & $\mathit{0.719}_{(0.211)}$ \\
nMW & $2.232_{(2.204)}$ & $2.828_{(3.052)}$ & $2.999_{(2.627)}$ & $1.280_{(1.869)}$ & $0.648_{(0.669)}$ & $2.155_{(1.495)}$ & $7.830_{(9.431)}$ \\
MDOT & $0.979_{(0.346)}$ & $0.652_{(0.359)}$ & $1.167_{(0.371)}$ & $0.411_{(0.354)}$ & $0.789_{(0.203)}$ & $0.772_{(0.136)}$ & $1.131_{(0.285)}$ \\
PSW & $1.313_{(0.221)}$ & $0.376_{(0.112)}$ & $\mathbf{0.422}_{(0.426)}$ & $\underline{0.016}_{(0.022)}$ & $0.925_{(0.100)}$ & $0.419_{(0.285)}$ & $\underline{0.716}_{(0.109)}$ \\
\midrule
\multicolumn{8}{l}{\textit{Ours}} \\
DRIO & $\mathbf{0.645}_{(0.193)}$ & $\underline{0.252}_{(0.105)}$ & $\underline{0.818}_{(0.337)}$ & $\mathit{0.024}_{(0.021)}$ & $\mathbf{0.097}_{(0.157)}$ & $\mathbf{0.256}_{(0.260)}$ & $0.773_{(0.271)}$ \\
\midrule
\multicolumn{8}{l}{\textbf{MMD$^2$ (joint feature-temporal distributional alignment, lower better) }} \\
\midrule
\multicolumn{8}{l}{\textit{Baselines}} \\
Mean & $0.208_{(0.229)}$ & $0.255_{(0.294)}$ & $0.231_{(0.247)}$ & $0.295_{(0.307)}$ & $0.201_{(0.215)}$ & $0.184_{(0.184)}$ & $0.177_{(0.180)}$ \\
MF & $0.202_{(0.205)}$ & $0.132_{(0.229)}$ & $0.179_{(0.176)}$ & $0.121_{(0.155)}$ & $0.175_{(0.181)}$ & $0.116_{(0.110)}$ & $0.065_{(0.066)}$ \\
\midrule
\multicolumn{8}{l}{\textit{Benchmarks}} \\
CSDI & $0.122_{(0.212)}$ & $0.171_{(0.093)}$ & $0.144_{(0.213)}$ & $\mathbf{0.000}_{(0.000)}$ & $0.041_{(0.072)}$ & $\mathit{0.014}_{(0.023)}$ & $\mathbf{0.021}_{(0.020)}$ \\
SSSD & $0.193_{(0.009)}$ & $0.052_{(0.061)}$ & $0.069_{(0.081)}$ & $0.058_{(0.012)}$ & $0.025_{(0.010)}$ & $0.048_{(0.075)}$ & $0.051_{(0.005)}$ \\
BRITS & $0.123_{(0.213)}$ & $0.161_{(0.279)}$ & $0.145_{(0.251)}$ & $0.009_{(0.016)}$ & $0.087_{(0.151)}$ & $0.143_{(0.185)}$ & $0.137_{(0.170)}$ \\
SAITS & $\mathit{0.102}_{(0.102)}$ & $\underline{0.015}_{(0.014)}$ & $\underline{0.050}_{(0.048)}$ & $0.002_{(0.002)}$ & $\underline{0.006}_{(0.010)}$ & $\underline{0.005}_{(0.007)}$ & $\underline{0.026}_{(0.025)}$ \\
IF & $0.104_{(0.054)}$ & $\mathit{0.015}_{(0.016)}$ & $0.068_{(0.026)}$ & $\mathit{0.000}_{(0.000)}$ & $\mathit{0.012}_{(0.020)}$ & $0.023_{(0.035)}$ & $0.049_{(0.064)}$ \\
nMW & $0.291_{(0.393)}$ & $0.353_{(0.462)}$ & $0.331_{(0.421)}$ & $0.268_{(0.464)}$ & $0.153_{(0.266)}$ & $0.315_{(0.403)}$ & $0.230_{(0.261)}$ \\
MDOT & $0.179_{(0.225)}$ & $0.179_{(0.311)}$ & $0.195_{(0.250)}$ & $0.164_{(0.284)}$ & $0.145_{(0.205)}$ & $0.146_{(0.171)}$ & $0.164_{(0.179)}$ \\
PSW & $\mathbf{0.051}_{(0.089)}$ & $0.033_{(0.057)}$ & $\mathit{0.062}_{(0.107)}$ & $\underline{0.000}_{(0.000)}$ & $0.047_{(0.081)}$ & $0.045_{(0.078)}$ & $0.029_{(0.026)}$ \\
\midrule
\multicolumn{8}{l}{\textit{Ours}} \\
DRIO & $\underline{0.087}_{(0.082)}$ & $\mathbf{0.015}_{(0.014)}$ & $\mathbf{0.039}_{(0.038)}$ & $0.001_{(0.002)}$ & $\mathbf{0.004}_{(0.007)}$ & $\mathbf{0.005}_{(0.008)}$ & $\mathit{0.026}_{(0.028)}$ \\
\bottomrule
\end{tabular}
\end{adjustbox}
\end{threeparttable}
\end{table*}

\begin{table*}[thb]
\centering
\caption{Per-dataset \ref{eq:WF} and \ref{eq:1dw2} distances under MCAR missingness. Each cell shows the mean with the standard deviation in parentheses, across all missing ratios. Data normalization procedure, highlighting, and $\objname$ backbone all follow same setting described in \cref{tab:ablation}.}
\label{tab:appendix_mcar_wf_w2_stacked}
\begin{threeparttable}

\begin{adjustbox}{max width=0.85\textwidth}
\begin{tabular}{l c c c c c c c}
\toprule
Method & CNN & PeMS & PM2.5 & Gas & CMAP & HAR & Physio \\
\midrule
\multicolumn{8}{l}{\textbf{WF (distributional alignment in frequency domain, lower better) }} \\
\midrule
\multicolumn{8}{l}{\textit{Baselines}} \\
Mean & $0.082_{(0.042)}$ & $0.124_{(0.077)}$ & $0.164_{(0.075)}$ & $0.177_{(0.105)}$ & $0.080_{(0.050)}$ & $0.142_{(0.085)}$ & $0.119_{(0.080)}$ \\
MF & $0.089_{(0.041)}$ & $0.097_{(0.081)}$ & $0.176_{(0.080)}$ & $0.111_{(0.056)}$ & $0.048_{(0.022)}$ & $0.178_{(0.053)}$ & $0.119_{(0.065)}$ \\
\midrule
\multicolumn{8}{l}{\textit{Benchmarks}} \\
CSDI & $0.054_{(0.041)}$ & $0.051_{(0.025)}$ & $0.055_{(0.034)}$ & $0.028_{(0.031)}$ & $\underline{0.009}_{(0.012)}$ & $\mathbf{0.015}_{(0.013)}$ & $0.039_{(0.031)}$ \\
SSSD & $0.115_{(0.017)}$ & $0.057_{(0.036)}$ & $0.056_{(0.008)}$ & $0.198_{(0.056)}$ & $0.087_{(0.042)}$ & $\mathit{0.025}_{(0.017)}$ & $0.064_{(0.032)}$ \\
BRITS & $0.045_{(0.032)}$ & $0.053_{(0.042)}$ & $0.040_{(0.028)}$ & $0.018_{(0.010)}$ & $0.055_{(0.069)}$ & $0.047_{(0.053)}$ & $\mathbf{0.036}_{(0.026)}$ \\
SAITS & $\mathbf{0.039}_{(0.026)}$ & $\mathit{0.036}_{(0.026)}$ & $\mathit{0.036}_{(0.019)}$ & $\mathit{0.012}_{(0.007)}$ & $\mathit{0.010}_{(0.010)}$ & $0.033_{(0.022)}$ & $\mathit{0.038}_{(0.029)}$ \\
IF & $\mathit{0.041}_{(0.024)}$ & $\mathbf{0.035}_{(0.023)}$ & $0.042_{(0.020)}$ & $\underline{0.007}_{(0.002)}$ & $0.014_{(0.016)}$ & $0.050_{(0.048)}$ & $0.042_{(0.035)}$ \\
nMW & $0.101_{(0.091)}$ & $0.086_{(0.044)}$ & $0.078_{(0.042)}$ & $0.044_{(0.021)}$ & $0.047_{(0.066)}$ & $0.094_{(0.039)}$ & $0.079_{(0.038)}$ \\
MDOT & $0.070_{(0.048)}$ & $0.083_{(0.077)}$ & $0.114_{(0.083)}$ & $0.089_{(0.096)}$ & $0.035_{(0.029)}$ & $0.095_{(0.060)}$ & $0.091_{(0.091)}$ \\
PSW & $0.090_{(0.071)}$ & $0.046_{(0.033)}$ & $\mathbf{0.015}_{(0.012)}$ & $\mathbf{0.003}_{(0.004)}$ & $0.078_{(0.063)}$ & $\underline{0.019}_{(0.014)}$ & $0.048_{(0.033)}$ \\
\midrule
\multicolumn{8}{l}{\textit{Ours}} \\
DRIO & $\underline{0.041}_{(0.029)}$ & $\underline{0.036}_{(0.027)}$ & $\underline{0.031}_{(0.019)}$ & $0.013_{(0.007)}$ & $\mathbf{0.008}_{(0.007)}$ & $0.030_{(0.024)}$ & $\underline{0.038}_{(0.028)}$ \\
\midrule
\multicolumn{8}{l}{\textbf{$\mathbf1$-D W2 (entry-wise distributional alignment lower better) }} \\
\midrule
\multicolumn{8}{l}{\textit{Baselines}} \\
Mean & $0.231_{(0.152)}$ & $0.279_{(0.193)}$ & $0.334_{(0.215)}$ & $0.341_{(0.231)}$ & $0.374_{(0.222)}$ & $0.395_{(0.339)}$ & $0.360_{(0.281)}$ \\
MF & $\underline{0.137}_{(0.094)}$ & $0.457_{(0.673)}$ & $0.315_{(0.215)}$ & $0.127_{(0.110)}$ & $0.188_{(0.113)}$ & $0.106_{(0.078)}$ & $\mathit{0.199}_{(0.174)}$ \\
\midrule
\multicolumn{8}{l}{\textit{Benchmarks}} \\
CSDI & $0.294_{(0.273)}$ & $1.602_{(1.428)}$ & $0.308_{(0.216)}$ & $0.377_{(0.126)}$ & $\mathit{0.059}_{(0.091)}$ & $0.100_{(0.139)}$ & $\underline{0.168}_{(0.130)}$ \\
SSSD & $1.472_{(0.932)}$ & $0.230_{(0.194)}$ & $\mathit{0.285}_{(0.185)}$ & $0.318_{(0.069)}$ & $0.306_{(0.125)}$ & $0.183_{(0.229)}$ & $0.472_{(0.250)}$ \\
BRITS & $0.247_{(0.253)}$ & $0.266_{(0.291)}$ & $0.314_{(0.228)}$ & $0.057_{(0.063)}$ & $0.193_{(0.216)}$ & $0.226_{(0.264)}$ & $0.250_{(0.241)}$ \\
SAITS & $0.217_{(0.228)}$ & $\mathit{0.091}_{(0.077)}$ & $0.292_{(0.200)}$ & $0.035_{(0.041)}$ & $\underline{0.056}_{(0.074)}$ & $\underline{0.068}_{(0.092)}$ & $0.208_{(0.180)}$ \\
IF & $0.204_{(0.184)}$ & $0.096_{(0.076)}$ & $0.293_{(0.196)}$ & $\mathit{0.027}_{(0.024)}$ & $0.149_{(0.149)}$ & $0.165_{(0.214)}$ & $0.210_{(0.185)}$ \\
nMW & $0.820_{(1.130)}$ & $0.990_{(1.252)}$ & $0.852_{(1.050)}$ & $0.780_{(1.194)}$ & $0.370_{(0.486)}$ & $0.753_{(1.006)}$ & $1.481_{(2.001)}$ \\
MDOT & $\mathit{0.189}_{(0.133)}$ & $0.196_{(0.154)}$ & $0.298_{(0.199)}$ & $0.197_{(0.196)}$ & $0.210_{(0.139)}$ & $0.249_{(0.240)}$ & $0.319_{(0.290)}$ \\
PSW & $\mathbf{0.100}_{(0.078)}$ & $\mathbf{0.046}_{(0.038)}$ & $\mathbf{0.114}_{(0.161)}$ & $\mathbf{0.008}_{(0.012)}$ & $0.156_{(0.125)}$ & $\mathit{0.085}_{(0.104)}$ & $\mathbf{0.139}_{(0.117)}$ \\
\midrule
\multicolumn{8}{l}{\textit{Ours}} \\
DRIO & $0.200_{(0.206)}$ & $\underline{0.080}_{(0.070)}$ & $\underline{0.278}_{(0.187)}$ & $\underline{0.027}_{(0.024)}$ & $\mathbf{0.052}_{(0.066)}$ & $\mathbf{0.063}_{(0.086)}$ & $0.208_{(0.180)}$ \\
\bottomrule
\end{tabular}
\end{adjustbox}
\end{threeparttable}
\end{table*}

\begin{table*}[thb]
\centering
\caption{Per-dataset \ref{eq:WF} and \ref{eq:1dw2} distances under MNAR missingness. Each cell shows the mean with the standard deviation in parentheses, across all missing ratios. Data normalization procedure, highlighting, and $\objname$ backbone all follow same setting described in \cref{tab:ablation}.}
\label{tab:appendix_mnar_wf_w2_stacked}
\begin{threeparttable}

\begin{adjustbox}{max width=0.85\textwidth}
\begin{tabular}{l c c c c c c c}
\toprule
Method & CNN & PeMS & PM2.5 & Gas & CMAP & HAR & Physio \\
\midrule
\multicolumn{8}{l}{\textbf{WF (distributional alignment in frequency domain, lower better) }} \\
\midrule
\multicolumn{8}{l}{\textit{Baselines}} \\
Mean & $0.082_{(0.044)}$ & $0.117_{(0.062)}$ & $0.143_{(0.064)}$ & $0.177_{(0.100)}$ & $0.077_{(0.049)}$ & $0.141_{(0.075)}$ & $0.125_{(0.077)}$ \\
MF & $0.090_{(0.043)}$ & $0.116_{(0.093)}$ & $0.179_{(0.079)}$ & $0.120_{(0.068)}$ & $0.049_{(0.028)}$ & $0.183_{(0.052)}$ & $0.124_{(0.067)}$ \\
\midrule
\multicolumn{8}{l}{\textit{Benchmarks}} \\
CSDI & $0.059_{(0.043)}$ & $0.056_{(0.031)}$ & $0.064_{(0.038)}$ & $0.025_{(0.021)}$ & $\mathbf{0.013}_{(0.018)}$ & $\mathbf{0.015}_{(0.012)}$ & $\mathit{0.041}_{(0.036)}$ \\
SSSD & $0.116_{(0.019)}$ & $0.057_{(0.048)}$ & $0.063_{(0.015)}$ & $0.199_{(0.078)}$ & $0.083_{(0.048)}$ & $\mathit{0.027}_{(0.025)}$ & $0.073_{(0.039)}$ \\
BRITS & $0.047_{(0.034)}$ & $0.059_{(0.053)}$ & $0.051_{(0.039)}$ & $0.019_{(0.012)}$ & $0.034_{(0.034)}$ & $0.057_{(0.065)}$ & $\mathbf{0.038}_{(0.029)}$ \\
SAITS & $\mathit{0.045}_{(0.032)}$ & $\mathit{0.040}_{(0.033)}$ & $\mathit{0.037}_{(0.020)}$ & $0.012_{(0.006)}$ & $\underline{0.013}_{(0.016)}$ & $0.038_{(0.027)}$ & $0.041_{(0.034)}$ \\
IF & $\underline{0.043}_{(0.022)}$ & $\mathbf{0.036}_{(0.025)}$ & $0.049_{(0.023)}$ & $\underline{0.006}_{(0.001)}$ & $0.017_{(0.023)}$ & $0.056_{(0.055)}$ & $0.042_{(0.037)}$ \\
nMW & $0.098_{(0.088)}$ & $0.083_{(0.040)}$ & $0.088_{(0.051)}$ & $0.044_{(0.023)}$ & $0.051_{(0.068)}$ & $0.093_{(0.045)}$ & $0.081_{(0.036)}$ \\
MDOT & $0.071_{(0.051)}$ & $0.085_{(0.080)}$ & $0.108_{(0.068)}$ & $0.093_{(0.099)}$ & $0.038_{(0.037)}$ & $0.103_{(0.069)}$ & $0.094_{(0.092)}$ \\
PSW & $0.091_{(0.069)}$ & $0.047_{(0.034)}$ & $\mathbf{0.017}_{(0.015)}$ & $\mathbf{0.003}_{(0.004)}$ & $0.076_{(0.062)}$ & $\underline{0.019}_{(0.014)}$ & $0.052_{(0.041)}$ \\
\midrule
\multicolumn{8}{l}{\textit{Ours}} \\
DRIO & $\mathbf{0.041}_{(0.028)}$ & $\underline{0.039}_{(0.033)}$ & $\underline{0.033}_{(0.018)}$ & $\mathit{0.012}_{(0.007)}$ & $\mathit{0.014}_{(0.018)}$ & $0.035_{(0.030)}$ & $\underline{0.040}_{(0.033)}$ \\
\midrule
\multicolumn{8}{l}{\textbf{$1$-D W2 (entry-wise distributional alignment lower better) }} \\
\midrule
\multicolumn{8}{l}{\textit{Baselines}} \\
Mean & $0.333_{(0.232)}$ & $0.314_{(0.212)}$ & $0.457_{(0.352)}$ & $0.435_{(0.292)}$ & $0.401_{(0.244)}$ & $0.458_{(0.356)}$ & $0.450_{(0.268)}$ \\
MF & $\underline{0.204}_{(0.121)}$ & $0.458_{(0.660)}$ & $\mathit{0.363}_{(0.266)}$ & $0.146_{(0.118)}$ & $0.218_{(0.136)}$ & $\mathit{0.128}_{(0.075)}$ & $\underline{0.231}_{(0.103)}$ \\
\midrule
\multicolumn{8}{l}{\textit{Benchmarks}} \\
CSDI & $0.385_{(0.344)}$ & $2.505_{(1.910)}$ & $0.422_{(0.377)}$ & $0.391_{(0.206)}$ & $\mathit{0.102}_{(0.162)}$ & $0.169_{(0.225)}$ & $\mathit{0.234}_{(0.156)}$ \\
SSSD & $1.434_{(0.914)}$ & $0.234_{(0.209)}$ & $0.393_{(0.338)}$ & $0.357_{(0.153)}$ & $0.289_{(0.155)}$ & $0.261_{(0.323)}$ & $0.495_{(0.217)}$ \\
BRITS & $0.356_{(0.356)}$ & $0.316_{(0.341)}$ & $0.436_{(0.411)}$ & $0.100_{(0.130)}$ & $0.245_{(0.283)}$ & $0.320_{(0.361)}$ & $0.364_{(0.272)}$ \\
SAITS & $0.311_{(0.315)}$ & $\underline{0.109}_{(0.080)}$ & $0.374_{(0.331)}$ & $0.034_{(0.034)}$ & $\underline{0.089}_{(0.127)}$ & $\underline{0.104}_{(0.132)}$ & $0.273_{(0.148)}$ \\
IF & $0.307_{(0.267)}$ & $0.148_{(0.108)}$ & $0.390_{(0.327)}$ & $\underline{0.025}_{(0.022)}$ & $0.147_{(0.184)}$ & $0.212_{(0.265)}$ & $0.319_{(0.221)}$ \\
nMW & $0.708_{(0.870)}$ & $0.893_{(1.131)}$ & $0.833_{(0.982)}$ & $0.574_{(0.862)}$ & $0.383_{(0.467)}$ & $0.647_{(0.710)}$ & $1.473_{(1.932)}$ \\
MDOT & $\mathit{0.289}_{(0.220)}$ & $0.250_{(0.208)}$ & $0.419_{(0.356)}$ & $0.279_{(0.279)}$ & $0.280_{(0.215)}$ & $0.328_{(0.305)}$ & $0.415_{(0.279)}$ \\
PSW & $\mathbf{0.192}_{(0.159)}$ & $\mathbf{0.093}_{(0.079)}$ & $\mathbf{0.228}_{(0.327)}$ & $\mathbf{0.017}_{(0.026)}$ & $0.213_{(0.175)}$ & $0.148_{(0.177)}$ & $\mathbf{0.210}_{(0.153)}$ \\
\midrule
\multicolumn{8}{l}{\textit{Ours}} \\
DRIO & $0.292_{(0.291)}$ & $\mathit{0.110}_{(0.082)}$ & $\underline{0.357}_{(0.323)}$ & $\mathit{0.030}_{(0.031)}$ & $\mathbf{0.070}_{(0.098)}$ & $\mathbf{0.101}_{(0.138)}$ & $0.274_{(0.157)}$ \\
\bottomrule
\end{tabular}
\end{adjustbox}
\end{threeparttable}
\end{table*}

\cref{tab:appendix_mcar_mse_mmd_stacked,tab:appendix_mnar_mse_mmd_stacked} report the main imputation metrics, MSE and MMD$^2$, under MCAR and MNAR. Across both missingness mechanisms, $\objname$ is among the most stable methods across datasets and metrics. Under MCAR, $\objname$ achieves the best MSE on CNNpred, PeMS08, CMAPSS, and HAR, and remains competitive on GasSensor and PhysioNet. Under MNAR, where systematic missingness induces stronger distributional bias, $\objname$ remains competitive among the baselines and benchmarks. Specifically, $\objname$ achieves the best MSE on CNNpred, CMAPSS, and HAR, the second-best MSE on PeMS08 and PM2.5, and the third-best MSE on GasSensor. These results suggest that in the cases where the observed empirical distribution can be a biased proxy for the true data-generating distribution, our proposed distributional robustness regularized objective provides robust and stable imputations. Importantly, the relatively small degradation from MCAR to MNAR suggests that explicitly incorporating distributional uncertainty helps stabilize imputation across missingness scenarios.

As for MMD$^2$ (which reflects joint feature-temporal alignment), we observe a similar pattern. Under MCAR, $\objname$ is consistently competitive, achieving top-three on CNNpred, PeMS08, PM2.5, CMAPSS, HAR, and PhysioNet. Under MNAR, $\objname$ achieves the best MMD$^2$ on PeMS08, PM2.5, CMAPSS, and HAR, and remains competitive on CNNpred and PhysioNet. These results indicate that $\objname$ does not merely reduce point-wise error, but also improves the distributional structure of the completed trajectories. 

The full tables also show that strong baselines vary considerably across datasets. CSDI performs very well on selected datasets, e.g., CMAPSS and PhysioNet, but is unstable on PeMS08, where its MSE becomes extremely large under both MCAR and MNAR. This suggests that diffusion-based generation can be sensitive to non-stationary traffic dynamics and structural missingness. SSSD also performs competitively in some settings, but its performance is highly dataset-dependent, with large errors on CNNpred and GasSensor. PSW, as an OT-based method for each sample, is particularly strong for one-dimensional distributional metrics and for datasets such as PM2.5 and GasSensor. This is expected because it directly optimizes a Wasserstein-style alignment objective for every sample. However, PSW is less uniform in MSE and does not consistently dominate on MMD$^2$. In contrast, while $\objname$ is not the best method on every dataset or metric, but it provides the most balanced trade-off across reconstruction accuracy, joint distributional alignment, and missingness mechanisms.

\cref{tab:appendix_mcar_wf_w2_stacked,tab:appendix_mnar_wf_w2_stacked} provide additional distributional evaluations using frequency-domain Wasserstein distance (WF) and one-dimensional W2. These metrics highlight additional aspects of the imputed series. For WF, $\objname$ is also consistently competitive: under MCAR, it achieves the best result on CMAPSS and top-tier results on CNNpred, PeMS08, PM2.5, and PhysioNet; under MNAR, it achieves the best result on CNNpred and remains among the top methods on PeMS08, PM2.5, GasSensor, CMAPSS, and PhysioNet. This suggests that the robust regularizer helps preserve spectral structure such as periodicity and temporal frequency content. For one-dimensional W2, PSW is often strongest, which is expected because it directly targets Wasserstein-type distributional alignment. That said, $\objname$ remains competitive on several datasets, especially PeMS08, PM2.5, GasSensor, CMAPSS, and HAR. Since one-dimensional W2 collapses sample, feature, and temporal structure into a scalar marginal distribution, these results as complementary rather than primary evidence of joint distributional alignment.

Finally, the competitiveness of CSDI, SSSD, and PSW should be interpreted together with computational cost. Diffusion-based methods such as CSDI and SSSD require iterative denoising or generative sampling procedures, while PSW relies on Sinkhorn optimization on each sample, which can be expensive when applied across many samples and scenarios. $\objname$ also introduces additional training cost through alternating adversarial updates, but these added costs are moderate, compared to CSDI, SSSD, and PSW. We provide wall-clock training comparisons in \cref{tab:run_time}.  Additionally, we note that an imputer trained with $\objname$ does not change the inference-time imputation architecture. Therefore, the inference-time costs are the same as training the same backbones with a non-robust objective like MSE. 

Overall, the results show that $\objname$ achieves a favorable robustness-accuracy trade-off as it consistently provides strong reconstruction, joint distributional alignment, and downstream-relevant robustness across diverse datasets and missingness mechanisms.

\subsubsection{More on Cross Validation}\label{sec:cv_compare}

\cref{tab:hyperparam_val_vs_oracle} compares the deployable cross-validation selection rule based on validation MSE (See \cref{sec:hyperparams} for details) with a non-deployable oracle that selects $(\alpha,\gamma)$ using test MSE. Recall that $\alpha$ controls the reconstruction--robustness trade-off in \eqref{eq:final-loss}: smaller $\alpha$ assigns larger weight to the worst-case Sinkhorn regularizer, while larger $\alpha$ keeps the objective closer to reconstruction training. The parameter $\gamma$ controls the transport-cost penalty in the inner maximization: smaller $\gamma$ allows the adversary to explore a broader ambiguity neighborhood, whereas larger $\gamma$ restricts the adversary closer to the empirical observations.

From \cref{tab:hyperparam_val_vs_oracle}, the oracle selected $\alpha \leq 0.75$ for more than $54\%$ across all cases, with consistent percentages across MCAR and MNAR scenarios, showing effectiveness of distributional robustness in MTS imputation. The oracle-selected $\alpha$ values tend to increase as the missing ratio grows. This is reasonable since at higher missingness, the imputer has very limited information, so an overly aggressive adversarial term can become less effective. That said, the selection criterion is MSE-based, which naturally favors configurations with stronger reconstruction emphasis. Even under these conditions, more than half of the settings prefer a high robust regularization weight. 

The behavior of $\gamma$ further suggests that the appropriate robustness radius is dataset-dependent. Across the table, both validation and oracle choices span the full grid $\gamma\in\{0.1,1,5,10\}$, indicating that no single adversarial radius dominates across datasets or missingness mechanisms. This is consistent with the role of $\gamma$: datasets with different temporal scales, feature correlations, and support mismatch require different degrees of adversarial exploration. Moreover, many disagreements in $\gamma$ do not translate into large performance differences, suggesting that several $\gamma$ values often lie on a similar performance plateau once $\alpha$ is reasonably chosen.

Notably, in these CV procedures the grid is intentionally small due to computational efficiency, with only four values of $\alpha$ and four values of $\gamma$. Hence, the reported performance should be viewed as a practical lower bound under a coarse deployable search and a finer grid cross validation criterion could further improve the selected configuration. Overall, these results support our goal of designing $\objname$ as a reconstruction objective regularized by distributional robustness, rather than as a purely reconstruction-driven or robust-driven method.

Finally, the deployable validation rule closely tracks the oracle in test performance. It exactly matches the oracle pair in $16/42$ scenarios ($38.1\%$) and selects the same $\alpha$ in $24/42$ scenarios ($57.1\%$), indicating that it often identifies the same reconstruction-robustness trade-off even when the selected adversarial penalty $\gamma$ differs. \cref{fig:cv_test_mse_gap} further shows that the practical cost of deployable validation is small. Specifically, the optimality gap, defined as the difference between the test MSE of the validation-selected pair and that of the oracle-selected pair, is close to zero for most datasets and missing ratios, with only a few outlying scenarios. Thus, even when the selected $(\alpha,\gamma)$ does not exactly match the oracle, it typically lies in a near-oracle region of the grid. These results suggest that our deployable cross-validation procedure remains effective without access to true missing values at test time.

\begin{table*}[thb]
\centering
\caption{Hyperparameter selection comparison for $\objname$ with SAITS backbone. Each cell shows cross-validated ($\alpha$, $\gamma$) / oracle ($\alpha$, $\gamma$), selected by validation MSE and test MSE, respectively. \textbf{Bold} cells indicate exact agreement between the two selection procedure. Grid: $\alpha \in \{0.25, 0.5, 0.75, 0.9\}$, $\gamma \in \{0.1, 1, 5, 10\}$.}
\label{tab:hyperparam_val_vs_oracle}
\small
\begin{tabular}{ll rrr}
\toprule
& & \multicolumn{3}{c}{Missing Ratio} \\
\cmidrule(lr){3-5}
Dataset & Mechanism & 10\% & 50\% & 90\% \\
\midrule
  \multirow{2}{*}{CNNpred} & MCAR & (0.75, 10) / (0.5, 0.1) & (0.5, 1) / (0.25, 1) & \textbf{(0.75, 5) / (0.75, 5)} \\
   & MNAR & (0.25, 10) / (0.75, 0.1) & \textbf{(0.9, 0.1) / (0.9, 0.1)} & (0.9, 10) / (0.75, 5) \\
\midrule
  \multirow{2}{*}{PeMS08} & MCAR & (0.75, 1) / (0.9, 5) & \textbf{(0.9, 10) / (0.9, 10)} & (0.9, 0.1) / (0.9, 10) \\
   & MNAR & (0.75, 5) / (0.9, 0.1) & (0.9, 5) / (0.9, 10) & (0.25, 10) / (0.5, 0.1) \\
\midrule
  \multirow{2}{*}{PM2.5} & MCAR & (0.5, 10) / (0.25, 1) & (0.75, 10) / (0.9, 0.1) & \textbf{(0.9, 1) / (0.9, 1)} \\
   & MNAR & (0.75, 1) / (0.9, 10) & (0.9, 0.1) / (0.75, 10) & \textbf{(0.9, 0.1) / (0.9, 0.1)} \\
\midrule
  \multirow{2}{*}{GasSensor} & MCAR & \textbf{(0.75, 0.1) / (0.75, 0.1)} & \textbf{(0.75, 0.1) / (0.75, 0.1)} & (0.9, 1) / (0.9, 10) \\
   & MNAR & \textbf{(0.5, 5) / (0.5, 5)} & (0.9, 0.1) / (0.75, 0.1) & \textbf{(0.9, 1) / (0.9, 1)} \\
\midrule
  \multirow{2}{*}{CMAPSS} & MCAR & \textbf{(0.5, 1) / (0.5, 1)} & (0.75, 0.1) / (0.9, 10) & \textbf{(0.9, 1) / (0.9, 1)} \\
   & MNAR & \textbf{(0.75, 0.1) / (0.75, 0.1)} & (0.75, 5) / (0.75, 10) & \textbf{(0.9, 0.1) / (0.9, 0.1)} \\
\midrule
  \multirow{2}{*}{HAR} & MCAR & \textbf{(0.5, 0.1) / (0.5, 0.1)} & (0.75, 0.1) / (0.75, 10) & \textbf{(0.9, 1) / (0.9, 1)} \\
   & MNAR & (0.25, 5) / (0.25, 10) & (0.9, 5) / (0.9, 10) & \textbf{(0.9, 0.1) / (0.9, 0.1)} \\
\midrule
  \multirow{2}{*}{PhysioNet} & MCAR & (0.5, 5) / (0.25, 0.1) & (0.9, 1) / (0.9, 5) & (0.75, 5) / (0.5, 5) \\
   & MNAR & (0.75, 5) / (0.25, 0.1) & (0.75, 10) / (0.5, 0.1) & (0.5, 10) / (0.25, 5) \\
\bottomrule
\end{tabular}
\end{table*}

\begin{figure}[thb]
\centering
\includegraphics[width=0.85\textwidth]{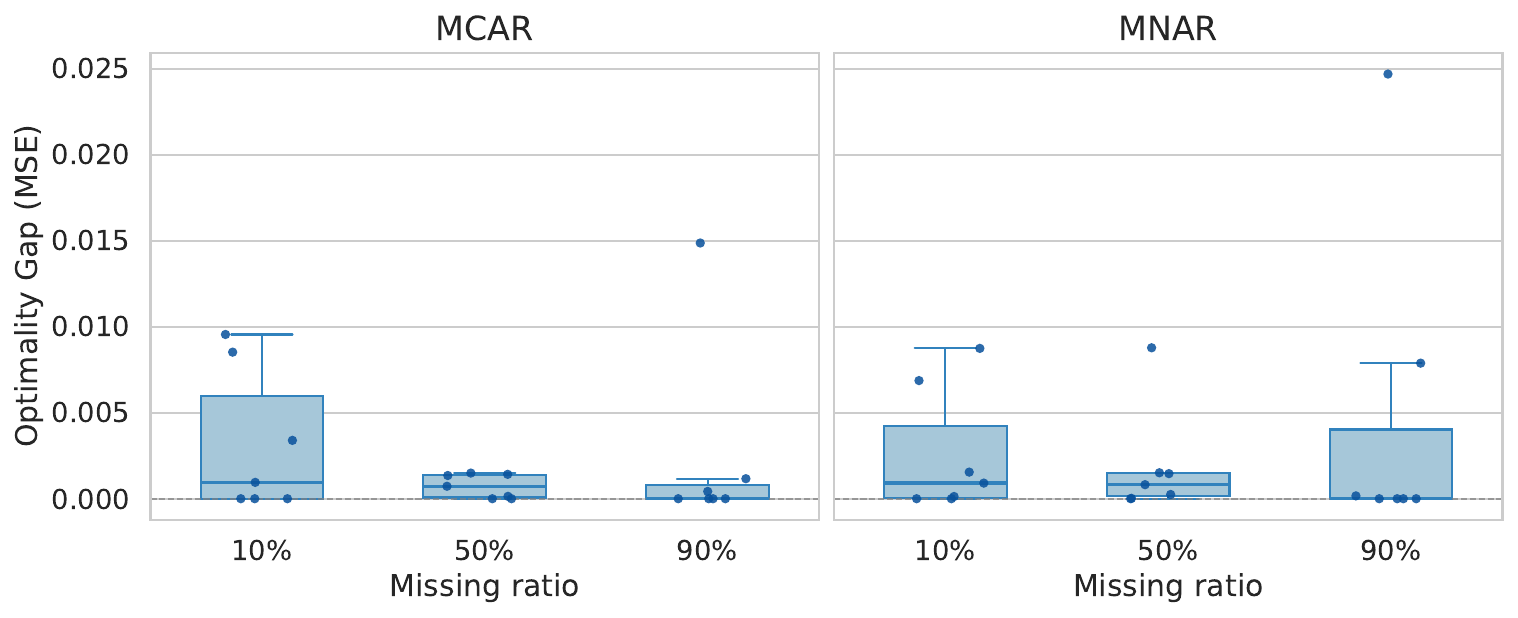}
\caption{Test-MSE gap between the deployable validation-MSE pick and the oracle test-MSE pick, for DRIO-SAITS. Each box aggregates seven datasets at a fixed (mechanism, missing ratio); individual datasets are shown as points. A gap of zero means the val-pick matches the oracle on test MSE.}
\label{fig:cv_test_mse_gap}
\end{figure}

\subsubsection{Sensitivity to Mass Relaxation}\label{sec:tau_sensitivity}
\cref{tab:appendix_tau_sensitivity} studies the effect of the unbalanced Sinkhorn marginal-relaxation parameter $\tau$. Smaller $\tau$ allows more local mass creation and destruction, while larger $\tau$ penalizes marginal mismatch more strongly. The balanced variant corresponds to $\tau\to\infty$, which enforces strict mass preservation. 

Our main observation is that finite mass relaxation is consistently better than balanced Sinkhorn. Under MCAR, the balanced variant increases MSE from $0.326$ to $0.361$, MMD$^2$ from $0.020$ to $0.041$, WF from $0.026$ to $0.030$, and $1$-D W2 from $0.127$ to $0.175$ compared with $\tau=10$. The degradation is even clearer under MNAR, where balanced Sinkhorn increases MSE from $0.367$ to $0.469$, MMD$^2$ from $0.029$ to $0.078$, WF from $0.028$ to $0.038$, and $1$-D W2 from $0.172$ to $0.268$. This supports our motivation that strict mass conservation is too restrictive for MTS imputation when there exists support mismatch between the observed empirical distribution and the true data-generating process, which can due to the combined effects of time series non-stationarity and structural missing mechanism.

Among finite values, performance is stable across $\tau\in\{1,5,10,15\}$. Under MCAR, the MSE range is only $0.326$--$0.331$, and MMD$^2$ ranges from $0.020$ to $0.022$. Under MNAR, the corresponding ranges are similarly small with MSE ranging from $0.367$ to $0.373$, MMD$^2$ from $0.029$ to $0.031$, WF from $0.028$ to $0.029$, and $1$-D W2 from $0.171$ to $0.174$. Therefore, we conclude that $\objname$ is not highly sensitive to the exact finite value of $\tau$; the important design choice is allowing unbalanced transport rather than enforcing balanced mass preservation.

We therefore fix $\tau=10$ throughout the paper as it achieves the best MSE, MMD$^2$, $1$-D W2 for MCAR, and MSE for MNAR, while remaining best or second-best on most remaining metrics. Overall, this choice permits enough marginal relaxation to account for missingness-induced support mismatch, while still penalizing excessive adversarial mass movement so that the worst-case distribution remains anchored to the empirical observations.

\begin{table*}[thb]
\centering
\caption{Sensitivity to the unbalanced Sinkhorn marginal-relaxation parameter $\tau$, under $\objname$ with SAITS backbone. Each cell reports the mean test metric with standard deviation in parentheses across all datasets and missing-ratio scenarios. ``Balanced'' denotes B-DRIO, i.e., ordinary Sinkhorn with strict mass preservation, which is equivalent to $\tau\to\infty$ in unbalanced Sinkhorn. The parameters $\alpha$ and $\gamma$ are fixed by the cross-validation results, while $\epsilon$ is computed using the heuristic described in \cref{sec:hyperparams}. \textbf{Bold}, \underline{underline}, and \textit{italic} denote the best, second-best, and third-best results, respectively. Lower is better for all metrics.}
\label{tab:appendix_tau_sensitivity}
\small
\setlength{\tabcolsep}{4pt}
\begin{adjustbox}{max width=\textwidth}
\begin{tabular}{l cccc cccc}
\toprule
& \multicolumn{4}{c}{MCAR} & \multicolumn{4}{c}{MNAR} \\
\cmidrule(lr){2-5} \cmidrule(lr){6-9}
Variant 
& MSE & MMD$^2$ & WF & $1$-D W2 
& MSE & MMD$^2$ & WF & $1$-D W2 \\
\midrule
$\tau = 1$ 
& $\underline{0.328}_{(0.381)}$ & $\underline{0.021}_{(0.032)}$ & $\mathbf{0.026}_{(0.020)}$ & $0.132_{(0.153)}$ 
& $\underline{0.368}_{(0.360)}$ & $0.030_{(0.048)}$ & $\underline{0.028}_{(0.022)}$ & $0.174_{(0.216)}$ \\
$\tau = 5$ 
& $0.331_{(0.381)}$ & $0.022_{(0.035)}$ & $0.026_{(0.021)}$ & $\underline{0.130}_{(0.155)}$ 
& $0.370_{(0.360)}$ & $\mathbf{0.029}_{(0.048)}$ & $0.028_{(0.022)}$ & $\mathbf{0.171}_{(0.216)}$ \\
$\tau = 10$ 
& $\mathbf{0.326}_{(0.381)}$ & $\mathbf{0.020}_{(0.030)}$ & $0.026_{(0.022)}$ & $\mathbf{0.127}_{(0.149)}$ 
& $\mathbf{0.367}_{(0.361)}$ & $\underline{0.029}_{(0.048)}$ & $\mathbf{0.028}_{(0.023)}$ & $\underline{0.172}_{(0.218)}$ \\
$\tau = 15$ 
& $0.329_{(0.381)}$ & $0.021_{(0.033)}$ & $\underline{0.026}_{(0.020)}$ & $0.132_{(0.154)}$ 
& $0.373_{(0.365)}$ & $0.031_{(0.051)}$ & $0.029_{(0.023)}$ & $0.173_{(0.221)}$ \\
Balanced ($\tau \to \infty$) 
& $0.361_{(0.392)}$ & $0.041_{(0.070)}$ & $0.030_{(0.026)}$ & $0.175_{(0.196)}$ 
& $0.469_{(0.436)}$ & $0.078_{(0.109)}$ & $0.038_{(0.032)}$ & $0.268_{(0.299)}$ \\
\bottomrule
\end{tabular}
\end{adjustbox}
\end{table*}

\subsubsection{Computational Complexity and Original-Scale Error}\label{sec:run_time}
\cref{tab:run_time} reports wall-clock training time together with original-scale, denormalized MSE. This table addresses two practical questions: how much computational overhead $\objname$ introduces, and whether the normalized MSE improvements translate back to the original data scale.

Compared with vanilla SAITS, training $\objname$ with the SAITS backbone incurs additional cost because each model update includes adversarial inner-loop optimization and Sinkhorn divergence computation. Across datasets, $\objname$ is about $5\times$ slower than SAITS on average. This overhead is expected, but it is incurred only during training as inference uses the same backbone and requires no adversarial updates. Moreover, we note that the absolute runtime remains practical; all datasets except HAR finish within 10 minutes with a single RTX-6000 GPU card.

Notably, the additional cost is accompanied by consistent original-scale accuracy gains. On average, $\objname$ improves over SAITS with a $12.8\%$ raw-MSE reduction, with gains exceeding $33\%$ on GasSensor and $28\%$ on CNNpred. $\objname$ also achieves top-three raw MSE on 6 of 7 datasets, including the best result on PeMS08, CMAPSS, and HAR. Thus, the manageable training cost provides robustness and accuracy gains in the original data scale, while preserving the same inference-time cost as the base SAITS model.

Compared to other benchmarks, $\objname$ is substantially more efficient than several competitive imputers. Diffusion-based methods such as CSDI and SSSD require expensive generative or denoising procedures, and PSW also involves costly OT-style optimization. In contrast, $\objname$ is faster than CSDI, SSSD, and PSW on every dataset. For example, on CMAPSS, DRIO takes $124$ seconds, compared with $5540$ seconds for CSDI, $1187$ seconds for SSSD, and $873$ seconds for PSW; on PhysioNet, DRIO takes $495$ seconds, compared with $6614$, $5647$, and $7158$ seconds, respectively. Therefore, while $\objname$ is more expensive than the plain SAITS backbone, it is far cheaper than many generative competitors while delivering strong robustness and raw-scale reconstruction performance.

These results clarify the computational trade-off of $\objname$, which introduces moderate training overhead relative to its base backbone, but this overhead in turn improves robustness and accuracy, and remains substantially below the cost of several strong benchmarks.  Overall, $\objname$ increases training cost, but preserves inference-time efficiency and provides a favorable accuracy--cost trade-off across diverse datasets.

\begin{table*}[thb]
\centering
\caption{Computational complexity and original-scale raw test MSE per dataset. Each cell is averaged across missingness mechanisms ratios, with standard deviation in parentheses. Runtime is wall-clock seconds for the training runs. For the MSE, \textbf{Bold}, \underline{underline}, and \textit{italic} denote the best, second-best, and third-best results, respectively.}
\label{tab:run_time}
\setlength{\tabcolsep}{1pt}
\begin{adjustbox}{max width=\textwidth}
\begin{tabular}{l c c c c c c c |c c c c c c c}
\toprule
& \multicolumn{7}{c}{Runtime (s) $\downarrow$} & \multicolumn{7}{c}{Raw MSE (denormalized) $\downarrow$} \\
\cmidrule(lr){2-8} \cmidrule(lr){9-15}
Method & CNN & PeMS & PM2.5 & Gas & CMAP & HAR & Physio & CNN & PeMS & PM2.5 & Gas & CMAP & HAR & Physio \\
\midrule
\multicolumn{15}{l}{\textit{Baselines}} \\
Mean & 0.29\,(0.04) & 0.02\,(0.00) & 0.08\,(0.00) & 0.16\,(0.01) & 0.33\,(0.03) & 0.21\,(0.02) & 0.31\,(0.03) & $2.00{\times}10^{6}_{(3.79{\times}10^{5})}$ & $3.53{\times}10^{8}_{(3.03{\times}10^{7})}$ & $1.72{\times}10^{3}_{(356)}$ & $0.284_{(0.040)}$ & $1.56{\times}10^{4}_{(1.04{\times}10^{3})}$ & $0.101_{(0.016)}$ & $1.24{\times}10^{4}_{(5.37{\times}10^{3})}$ \\
MF & 12\,(3.2) & 6.0\,(0.41) & 7.1\,(1.1) & 9.3\,(2.9) & 16\,(6.1) & 26\,(15) & 13\,(5.6) & $2.70{\times}10^{6}_{(5.84{\times}10^{5})}$ & $5.85{\times}10^{8}_{(6.72{\times}10^{8})}$ & $2.49{\times}10^{3}_{(946)}$ & $0.155_{(0.060)}$ & $1.44{\times}10^{4}_{(2.28{\times}10^{3})}$ & $0.162_{(0.016)}$ & $3.65{\times}10^{4}_{(1.44{\times}10^{4})}$ \\
\midrule
\multicolumn{15}{l}{\textit{Benchmarks}} \\
CSDI & 2696\,(175) & 522\,(4.4) & 1446\,(38) & 2828\,(82) & 5540\,(46) & 6056\,(3656) & 6614\,(603) & $3.02{\times}10^{5}_{(3.43{\times}10^{5})}$ & $2.08{\times}10^{9}_{(3.38{\times}10^{9})}$ & $1.45{\times}10^{3}_{(361)}$ & $0.336_{(0.349)}$ & $\mathit{1.27{\times}10^{3}}_{(2.05{\times}10^{3})}$ & $\mathit{0.019}_{(0.017)}$ & $\underline{6.76{\times}10^{3}}_{(3.52{\times}10^{3})}$ \\
SSSD & 611\,(253) & 526\,(52) & 824\,(369) & 727\,(23) & 1187\,(592) & 4143\,(49) & 5647\,(101) & $2.56{\times}10^{7}_{(2.42{\times}10^{6})}$ & $8.59{\times}10^{7}_{(5.68{\times}10^{7})}$ & $1.19{\times}10^{3}_{(661)}$ & $0.559_{(0.180)}$ & $2.49{\times}10^{4}_{(4.47{\times}10^{3})}$ & $0.028_{(0.031)}$ & $1.45{\times}10^{5}_{(1.11{\times}10^{4})}$ \\
BRITS & 47\,(3.2) & 45\,(3.4) & 167\,(1.7) & 184\,(16) & 285\,(19) & 1373\,(55) & 696\,(5.2) & $5.39{\times}10^{5}_{(5.85{\times}10^{5})}$ & $7.67{\times}10^{7}_{(7.68{\times}10^{7})}$ & $\mathit{763}_{(758)}$ & $0.035_{(0.029)}$ & $4.05{\times}10^{3}_{(4.09{\times}10^{3})}$ & $0.031_{(0.030)}$ & $9.77{\times}10^{3}_{(4.95{\times}10^{3})}$ \\
SAITS & 9.3\,(0.32) & 8.8\,(0.09) & 19\,(1.1) & 19\,(0.40) & 32\,(2.5) & 113\,(1.7) & 63\,(9.0) & $1.43{\times}10^{5}_{(1.93{\times}10^{5})}$ & $\underline{2.55{\times}10^{7}}_{(1.98{\times}10^{7})}$ & $825_{(633)}$ & $0.009_{(0.010)}$ & $\underline{958}_{(1.45{\times}10^{3})}$ & $\underline{0.017}_{(0.015)}$ & $8.90{\times}10^{3}_{(4.54{\times}10^{3})}$ \\
IF & 43\,(2.7) & 9.3\,(0.22) & 23\,(2.0) & 40\,(7.4) & 84\,(1.4) & 234\,(13) & 395\,(25) & $\underline{2.81{\times}10^{4}}_{(3.73{\times}10^{4})}$ & $3.07{\times}10^{7}_{(2.04{\times}10^{7})}$ & $1.34{\times}10^{3}_{(339)}$ & $\underline{0.004}_{(0.002)}$ & $1.76{\times}10^{3}_{(2.29{\times}10^{3})}$ & $0.029_{(0.018)}$ & $\mathit{7.43{\times}10^{3}}_{(4.47{\times}10^{3})}$ \\
nMW & 18\,(1.7) & 12\,(1.1) & 19\,(1.8) & 22\,(0.04) & 35\,(0.53) & 112\,(13) & 110\,(8.7) & $6.30{\times}10^{6}_{(8.33{\times}10^{6})}$ & $8.14{\times}10^{8}_{(7.97{\times}10^{8})}$ & $5.26{\times}10^{3}_{(5.30{\times}10^{3})}$ & $0.447_{(0.615)}$ & $8.12{\times}10^{3}_{(8.69{\times}10^{3})}$ & $0.217_{(0.187)}$ & $8.24{\times}10^{4}_{(7.31{\times}10^{4})}$ \\
MDOT & 130\,(3.5) & 103\,(4.6) & 129\,(9.3) & 136\,(3.1) & 149\,(2.0) & 264\,(14) & 295\,(10) & $1.64{\times}10^{6}_{(5.02{\times}10^{5})}$ & $1.55{\times}10^{8}_{(1.28{\times}10^{8})}$ & $1.52{\times}10^{3}_{(401)}$ & $0.103_{(0.069)}$ & $1.00{\times}10^{4}_{(2.96{\times}10^{3})}$ & $0.053_{(0.015)}$ & $1.29{\times}10^{4}_{(6.20{\times}10^{3})}$ \\
PSW & 305\,(27) & 399\,(65) & 589\,(31) & 608\,(75) & 873\,(53) & 6968\,(320) & 7158\,(153) & $\mathbf{1.65{\times}10^{4}}_{(3.82{\times}10^{4})}$ & $\mathit{2.64{\times}10^{7}}_{(1.22{\times}10^{7})}$ & $\mathbf{458}_{(546)}$ & $\mathbf{0.003}_{(0.004)}$ & $1.20{\times}10^{4}_{(1.60{\times}10^{3})}$ & $0.025_{(0.018)}$ & $\mathbf{6.40{\times}10^{3}}_{(3.11{\times}10^{3})}$ \\
\midrule
\multicolumn{15}{l}{\textit{Ours}} \\
DRIO & 53\,(16) & 36\,(1.9) & 88\,(3.8) & 103\,(31) & 124\,(25) & 845\,(301) & 495\,(181) & $\mathit{1.02{\times}10^{5}}_{(1.02{\times}10^{5})}$ & $\mathbf{2.45{\times}10^{7}}_{(2.35{\times}10^{7})}$ & $\underline{654}_{(575)}$ & $\mathit{0.006}_{(0.007)}$ & $\mathbf{926}_{(1.44{\times}10^{3})}$ & $\mathbf{0.017}_{(0.015)}$ & $8.93{\times}10^{3}_{(4.56{\times}10^{3})}$ \\
\bottomrule
\end{tabular}
\end{adjustbox}
\end{table*}